\documentclass[10pt]{article}

\usepackage{ulem}
\usepackage{amssymb,amsbsy,amsmath,amsfonts,amssymb,amscd,amsthm}
\usepackage{xcolor}
\usepackage{graphicx} 
\usepackage{multicol}
\graphicspath{{./}{figures/}} 
\usepackage{tikz}
\usetikzlibrary{matrix}

\newcommand{\R}{\mathbb{R}}
\newcommand{\PP}{\mathbb{P}}
\newcommand{\1}{\mathbbm{1}}

\newcommand{\N}{\mathbb{N}}

\newcommand{\mS}{\mathbb{S}}

\newcommand{\mR}{\mathcal R}
\newcommand{\mX}{\mathcal X}

\newcommand{\beq}{\begin{equation}}
	\newcommand{\eeq}{\end{equation}}
\def\a{\alpha}
\def\b{\beta}
\def\d{\delta}

\def\g{\gamma}
\def\l{\lambda}

\def\m{\mu}

\def\th{\theta}

\def\e{\varepsilon}

\newcommand{\cD}{{\cal D}}

\newcommand{\cN}{{\cal N}}

\newcommand{\cL}{{\cal L}}
\newcommand{\cP}{{\cal P}}

\newcommand{\cS}{{\cal S}}

\newcommand{\maS}{\mathbb{S}^S}
\newtheorem{theorem}{Theorem}[section]
\newtheorem{lemma}[theorem]{Lemma}
\newtheorem{definition}[theorem]{Definition}
\newtheorem{proposition}[theorem]{Proposition}

\newtheorem{remark}[theorem]{Remark}

\numberwithin{equation}{section}
\title{A Mean Field Games model for finite mixtures of Bernoulli and Categorical distributions}
\date{} 
\author{Laura Aquilanti, Simone Cacace, Fabio Camilli\\ and Raul De Maio}

\begin{document}
	\maketitle
	\begin{abstract}
		Finite mixture models are an important tool in the statistical analysis of data, for example in data clustering. 
		The optimal parameters of a mixture model are usually computed by maximizing the log-likelihood functional via the Expectation-Maximization algorithm. We propose an alternative approach based on the theory of Mean Field Games, a class of differential games with an infinite number of agents. We show that the solution of a finite state space multi-population Mean Field Games system characterizes the critical points of the   log-likelihood functional for a Bernoulli mixture. The approach is then  generalized to mixture models of categorical distributions.  Hence, the Mean Field Games approach  provides a  method to compute the parameters of the mixture model, and we show its application to some standard examples in cluster analysis.
	\end{abstract}
	
	\noindent
	{\footnotesize \textbf{AMS-Subject Classification:} 62H30, 60J10, 49N70, 91C20}.

	%%%%%%%%%%%%%%%%%%%%%%%%%
	%   Introduction        %
	%%%%%%%%%%%%%%%%%%%%%%%%%
	
	\section{Introduction}
	Finite mixture models, given by convex combinations of probability density functions  (PDFs in short) 
	\begin{equation}\label{eq:intro_mix}
		\pi(x)=\sum_{k=1}^K \a_k \pi_k(x),   \quad \text{with $\a_k\in [0,1]$, \ $\sum_{k=1}^K \a_k=1$},
	\end{equation}
	are an important mathematical tool in  statistical analysis of data. Introduced   by the biometrician K. Pearson  \cite{Pearson}, due to their flexibility,  they are employed in a large variety of fields as  astronomy, biology, genetic, medicine, marketing and engineering (see \cite[Chapter 6-7]{Wedel}, \cite{Titterington}).  In cluster analysis,  a classical problem in unsupervised Machine Learning  consisting in the repartition  of   a data set into subgroups with similar characteristics, finite mixture models  can be used in order to determine the intrinsic structure of clustered data when no information, except for the observed values, are available. For a detailed description of the theory of mixture models and applications, see \cite{Everitt,McLachlan,Titterington}.\par
	Given a data set $\mX$  representing the measurements of a phenomena, whose different values are related to the membership to  unknown categories, a corresponding  finite mixture model  is built by assuming that the data have been generated    by a random variable  $X$, whose unknown probability distribution $\pi$ can be described as in \eqref{eq:intro_mix}.    \par
	The   parameters of the mixture  \eqref{eq:intro_mix} are in general unknown and the aim is to determine them in such a way that they optimally fit the given data set $\mX$. To this end, different methods can be employed such as  the graphical method (\cite{tarter}), the  Bayesian method (\cite{Titterington}) and the  likelihood estimator (\cite{bishop}). The latter one, which is the starting point of our analysis, generates a tough quantity to be maximized,  which is usually computed by means of the  Expectation Maximization (EM in short) algorithm (see  \cite{bilmes}). In our approach, we characterize the optimal parameters of a mixture model through   a multi-population  Mean Field Games (MFG in short) system, a   coupled system  of differential or difference equations  which characterizes the Nash equilibria in the framework of  stochastic   games with a very large number of agents.
	The MFG theory has been introduced   simultaneously by Lasry-Lions \cite{Lions} and Huang-Caines-Malhamé \cite{HMC}  and it has been   successfully applied to different fields, such as economics, biology, environmental policy, etc. (for a plain introduction, see \cite{gomes_saude}). Recently, its scope has been broadened to Machine Learning  applications. In \cite{carmona_lauriere}, it is proposed a numerical scheme for  MFG problems based on tools from neural networks, whereas in \cite{E_et_altri}  the authors recast some deep learning techniques as   mean-field optimal control problems. Concerning unsupervised Machine Learning,  we refer to \cite[Chapter 2]{coron} and \cite{pequito} where multi-population MFG systems  are studied in connection with  cluster problems.
	\par
	In  \cite{aquilanti_et_altri}, we developed a MFG approach to finite mixture models  defined over  a {\it continuous} random variable $X$. Data points are interpreted as agents  and  the aim  is to subdivide  the whole population,  described by a PDF $\pi(x)$, in    $K$   sub-populations, described by   PDFs $\pi_k$ and   mixing coefficients $\a_k$,  on the basis of some similar characteristics  appropriately encoded in the cost functional of the control problems of the different populations. As a result, we end up with a multi-population ergodic Mean Field Game in $\mathbb{R}^d$ which, in the particular case of  a quadratic cost,   characterizes the critical points of   the log-likelihood functional for a mixture of Gaussian distributions.\par
	In this paper, we focus  on  a mixture model  for a {\it discrete} random variable $X$   described by a Bernoulli distribution or, more generally, by a categorical one. We assume   that the sub-populations
	can be discriminated on the basis of a single specific characteristic taking only one of $S$ different values (the case of several different characteristics will be discussed further on  in the paper). 
	\par
	Given a data set $\mX=\{x_1,\dots,x_N\}$, in order to characterize the components of the mixture model, i.e. the PDFs $\pi_k$ and the weights $\a_k$,
	we introduce the   $K$-populations finite state MFG system 
	\begin{equation}\label{eq:intro_MFG}
		\left\{
		\begin{array}{ll}
			V_k(i)=\displaystyle{\min_{P_i:\,P_{ij}\ge 0, \sum_j P_{ij}=1}}\left\{\sum_{j=1}^S P_{ij} \left(c ( P_{ij} )+\e\log(P_{ij})+F(i, \theta_k)+V_k(j)\right)\right\}-\l_k,\\ [8pt]
			\pi_k(i)= \sum_{ j=1}^S   P^{k}_{ji}\pi_k(j),\\[8pt]
			\pi_k(i)\ge 0,\, \sum_{i=1}^S\pi_k(i)=1, \,\sum_{i=1}^S V_k(i)=0,\\[8pt]
			\alpha_k=\frac{1}{N}\sum_{n=1}^N\g_k(x_n),
		\end{array}
		\right.
	\end{equation}
	for $i\in\{1,\dots,S\}$. 
	The term $c(P_{ij})$  is a transition cost between the  states $i$ and $j$ and $\e\log(P_{ij})$ is an entropy penalization term    which forces  the agents to diversify their transition choices. 
	Each of the $K$ sub-populations is characterized by a quadruple $(V_k,\l_k,\pi_k,\a_k)$, with the couple $(V_k,\l_k)$  solving an ergodic Hamilton-Jacobi-Bellman equation and with the probability distribution $\pi_k$  solving a Fokker-Planck equation where the transition matrix $P^{k}$,  composed of the rows $ P^{k}_i=\{P^{k}_{ij}\}_{j=1}^S$, is such that
	\begin{equation*}
	P^k_i=\displaystyle{{\arg\min}_{P_i:\,P_{ij}\ge 0, \sum_j P_{ij}=1}}\left\{\sum_{j=1}^S P_{ij} \left(c ( P_{ij} )+\e\log(P_{ij})+F(i, \theta_k)+V_k(j)\right)\right\}.
	\end{equation*}
	The vector $\theta_k\in\R^S$    represents the average value of the data set  with respect to the distribution $\pi_k$ and therefore depends on the solution of \eqref{eq:intro_MFG}.
	Interaction among the sub-populations is   encoded  in the weights $\a_k$ and in the coupling cost $F(i, \theta_k)$, which depend  on the responsibilities
	\begin{equation*}%\label{eq:intro_resp}
		\g_k(x_n)=\frac{\a_k\pi_k(x_n)}{\pi(x_n)}\qquad k=1,\dots,K,\,x_n\in\mX.
	\end{equation*}
	These quantities play  a crucial role and, in  cluster analysis,  can be  used to assign a point to the class with the highest $\g_k$ (see Section \ref{sec:mixture_models} for  some  more details on the game theoretic interpretation of \eqref{eq:intro_MFG}).\\
	Relying on the theory for finite states MFG system developed in  \cite{gms}, we prove that the system \eqref{eq:intro_MFG} admits a solution $(V_k,\lambda_k,\pi_k,\a_k)$, $k=1,\dots,K$. Moreover we show that, for $S=2$,  \eqref{eq:intro_MFG} characterizes the critical points of the maximum log-likelihood functional for a mixture of Bernoulli distributions and therefore gives an alternative way to compute the critical points of the log-likelihood functional. We show an  application  of this method to the computation of the optimal parameters for  mixture models related to    some standard  examples in cluster analysis, such as  digit classification.\par
	We remark that the goal of the paper is to provide a new  perspective to a class of problems which, up to now,  have been considered in the classical framework of finite dimensional optimization. For a    linear cost $c$ and $\e=0$, the  algorithm described in Section \ref{sec:numerics} is equivalent to the classical EM algorithm and, in general, it is not computationally competitive with other algorithms for cluster analysis. But our model is feasible to be generalized in several
	directions and, in a  forthcoming paper, we plan to   study  more general costs, possibly depending on the state variable, in order to exploit in a deeper way the structure of the data set.\par
	The paper is organized as follows. In Section \ref{sec:mixture_models}, we briefly review the finite mixture model  theory, with the corresponding EM algorithm, and the theory of finite state discrete time Mean Field Games problems. In Section \ref{sec:bernoulli}, we introduce the MFG model for a Bernoulli random variable and we show the connection with the maximization of the log-likelihood functional.
	In Section \ref{sec:categorical}, we generalize the model of the previous  section to a mixture of categorical distributions and we prove existence of a solution to the MFG system. In Section \ref{sec:numerics}, we apply   the MFG model to  some standard problems in cluster analysis.
	In Appendix \ref{sec:appendix},  we prove some results for a finite-state stationary MFG system we use in the previous sections.

	%%%%%%%%%%%%%%%%%%%%%%%%%%%%%%%%%%%%%%%%%%%%%%%%% 
	% Mixture Models & MFG:   a short introduction        %
	%%%%%%%%%%%%%%%%%%%%%%%%%%%%%%%%%%%%%%%%%%%%%%%%%      
	\section{A short introduction to mixture models and finite state discrete time Mean Field Games  problems}\label{sec:mixture_models}
	In this section, we   review  the parametric mixture model and the corresponding EM algorithm for the optimization of the parameters. We also give a short  description of the game theoretic interpretation of the MFG system \eqref{eq:intro_MFG}.
	\subsection*{Mixture models and EM algorithm.} Let $X$ be a random variable, univariate or multivariate and consider a sample $\mX=\{x_1, \dots, x_N \}$ of size $N$ of $X$,  where the sample space $\cS$ can be discrete or continuous. Let  $p(x)$ be the unknown distribution   of $X$, defined with respect to an appropriate reference  measure on $\mathcal{S}$.
	We assume that $X$ comes from a finite mixture model, i.e. $p(x)$ can be written as a convex combination of PDFs $p_k$ as
	\begin{equation}\label{eq:mix_mixture}
		p(x)= \sum_{k=1}^{K}\alpha_kp_k(x),\quad x\in\cS,
	\end{equation}
	where $K$, the number of  the  components of $p$, is supposed to be known a priori and $\alpha_k$,  the weights or mixing coefficients,   satisfy $\sum_{k=1}^{K}\alpha_k=1$ and $\alpha_k \geq 0$. How to determine the number $K$ of PDFs in the mixture model is an  unresolved issue and, in general,   a combination of criteria and experimental analysis is used to guide the decision (see \cite[Chapter 6]{Everitt}).\par
	Usually, it is assumed that the components of the mixture \eqref{eq:mix_mixture} belong to the same parametric family of density distributions, i.e. they can be written as $p_k(x) = p(x; \Theta_k)$, where $\Theta_k$ is the parameter which defines the $k$-th PDF. For example, in the Gaussian mixture model, $p (x;\Theta_k)=\cN(x;\m_k,\Sigma_k)$ are Gaussian distributions of parameters $\Theta_k=(\m_k,\Sigma_k)$ where $\mu_k$, $\Sigma_k$ are the mean and  covariance matrix; in the Bernoulli mixture model, $p(x; \Theta_k)= \mathcal{B}(x; \mu_k)$ are Bernoulli distributions of parameter $\Theta_k = \mu_k$. 
	The  aim is to find the parameters $\alpha=(\a_1,\dots,\a_K)$ and $\Theta=(\Theta_1,\dots,\Theta_K)$ which give  the best representation of the sample $\mX$. This can be achieved 
	through the maximization of the log-likelihood functional
	\begin{equation}\label{eq:mix_log}
		\cL(\alpha,\Theta; \mX)=\sum_{n=1}^{N}\log  \bigg[\sum_{k=1}^{K}\alpha_kp(x_n ; \Theta_k)\bigg] 
	\end{equation}
	By writing the necessary condition for the extrema of \eqref{eq:mix_log}, we have
	\begin{align}
		&\dfrac{\partial \cL}{\partial \alpha_k}= \sum_{n=1}^{N}\dfrac{p(x_n; \Theta_k)}{\sum_{j=1}^{K}\alpha_jp(x_n;\Theta_j)}-\lambda=0, \label{alpha}\\
		&\dfrac{\partial \cL}{\partial \Theta_{k}}=\sum_{n=1}^{N} \dfrac{\alpha_k}{\sum_{j=1}^{K}\alpha_jp(x_n;\Theta_j)}\cdot\frac{\partial p(x_n;\Theta_k)}{\partial \Theta_{k}}=0.\label{theta}
	\end{align} 
	where $\l$ is a Lagrange multiplier which takes into account  the constraint for the mixing coefficients.
	The Expectation-Maximization algorithm is an iterative procedure  for the computation of a solution of the previous system.
	We shortly describe its derivation. Firstly, it is introduced a latent, or hidden, $K$-dimensional random variable $Y=(Y_1,\dots,Y_K)$, with $Y_k\in\{0,1\}$ and $\sum_{k=1}^KY_k=1$, saying which component of the mixture \eqref{eq:mix_mixture} has generated a given  sample point $x_n$.
	More specifically,  we assume that, for each observed point $x_n\in\R^D$, there exists a corresponding unobserved one $y_n\in\R^K$  such that, if the point $x_n$ has been generated by the $k$-component  of the mixture, then  $y_{n,k}=1$ and $y_{n,j}=0$ for $j\neq k$. Hence $Y$ is a multinomial random variable and we assume that $\PP(Y_k=1)=\a_k$.
	A simple application of Bayes' Theorem allows to compute   the responsibility, i.e. the probability of $Y$ given $X$,  
	\begin{equation}\label{eq:mix_resp}
		\gamma_k(x_n):=\mathbb{P}(Y_k=1|X=x_n)=\dfrac{\alpha_k p(x_n;\Theta_k)}{\sum_{j=1}^{K}\alpha_jp(x_n;\Theta_j)}.
	\end{equation}
	Multiplying \eqref{alpha} by $\alpha_k$ and summing over $k$ we get that $\l=N$, hence replacing \eqref{eq:mix_resp} in \eqref{alpha} we get a first condition
	for an extremum  of \eqref{eq:mix_log}
	\begin{equation}\label{eq:mix_critical1}
		\alpha_k= \dfrac{1}{N}\sum_{n=1}^{N} \gamma_k(x_n).
	\end{equation}
	Using \eqref{eq:mix_resp} in \eqref{theta}, we have
	\begin{equation}\label{eq:mix_critical2}
		\sum_{n=1}^{N} \g_k(x_n)\dfrac{\partial \log (p(x_n;\Theta_k))}{\partial \Theta_{k}}=0.
	\end{equation}
	Since the   coefficients $\alpha_k$ in \eqref{eq:mix_critical1} depend  on $\Theta_k$ via the responsibilities $\gamma_k$, the equations \eqref{eq:mix_critical1}-\eqref{eq:mix_critical2}  do not provide a closed form solution for the parameters.
	
	 However, they  suggest the following iterative scheme which alternates two steps. Starting from some arbitrary initialization $(\a^{(0)},\Theta^{(0)})$, at $h$-iteration we perform the following steps: in the Expectation step (E-step in short), using the current values $(\alpha^{(h)}, \Theta^{(h)})$ of the parameters, we compute the responsibilities  by means of formula \eqref{eq:mix_resp} with $\Theta_k^{(h)}$ and $\alpha_k^{(h)}$ in place of $\Theta_k$ and $\alpha_k$. In the Maximization step (M-step in short), given the responsibilities as in the E-step, we compute the new parameters $(\alpha^{(h+1)}, \Theta^{(h+1)})$ by means of \eqref{eq:mix_critical1}-\eqref{eq:mix_critical2}. It can be proved that at each iteration (E-step, M-step), the log-likelihood function increases its value. For more details see \cite{bilmes}, \cite[Chapter 9]{bishop}.\\
	For specific parametrized families of distributions, such  as  Gaussian and Bernoulli distributions, it is possible to get explicit formulas for the parameters $\Theta_k$.
	Computed  the responsibility $\gamma^{(h+1)}_k(x)$ by means of \eqref{eq:mix_resp} in the E-step, 
	in the M-step we have the following explicit formulas 
	\[
	\alpha^{(h+1)}_k= \dfrac{1}{N}\sum_{n=1}^{N}\gamma^{(h+1)}_k(x_n)
	\]
	and, for a Gaussian mixture,
	\begin{align*}
		&\mu_k^{(h+1)}=\dfrac{\sum_{n=1}^{N}\gamma^{(h+1)}_k(x_n)x_n}{\sum_{n=1}^{N}\gamma^{(h+1)}_k(x_n)}, \\
		& \Sigma^{(h+1)}_k= \dfrac{\sum_{n=1}^{N}\gamma^{(h+1)}_k(x_n)(x_n-\mu^{(h+1)}_k)(x_n-\mu_k^{(h+1)})^t}{\sum_{n=1}^{N}\gamma^{(h+1)}_k(x_n)},
	\end{align*}
	for a Bernoulli mixture,
	\begin{equation*}
		\mu_k^{(h+1)}=\dfrac{\sum_{n=1}^{N}\gamma^{(h+1)}_k(x_n)x_n}{\sum_{n=1}^{N}\gamma^{(h+1)}_k(x_n)}.
	\end{equation*}
		\begin{remark}
		 Equations \eqref{eq:mix_critical1}- \eqref{eq:mix_critical2} also characterize the necessary conditions  for the extrema of the functional 
		 \begin{equation}\label{eq: expected_likelihood}
		  \mathcal{\tilde{L}}=\sum_{n=1}^{N} \sum_{k=1}^{K} \gamma_k(x_n)(\log(\alpha_k)+ \log p(x_n; \Theta_k)),
		\end{equation}
		 where $\gamma_k$ is as in \eqref{eq:mix_resp}, which  provides a different, but equivalent approach to the optimization of the mixture model for a given data set (see \cite[Chapter 9.3]{bishop}).  
	
		\end{remark}
	%%%%%%%%%%%%  MFG %%%%%%%%%%%%%%%%%%%%%%%%%%%
	
\subsection*{Mean Field Games} 
 %\red{
 Following   \cite{biswas,gms}, we  give a short interpretation of   system \eqref{eq:intro_MFG} in terms of a control problem for a distribution of agents. For simplicity of notations, we will consider the case of a single population, i.e. $K=1$.\\
We first describe the finite horizon problem. Given two natural numbers $S,T\ge 1$,  representing respectively the number of possible states for the agents and the total duration of the process, a solution of the MFG problem is a sequence of pairs
$\{(\pi^t,V^t)\}_{t=0}^T$ where $\pi^t\in\R^S$, $\pi^t(i)\in[0,1]$ and $\sum_{i=1}^S \pi^t(i)=1$, is a  probability vector describing the distribution of the agents among the $S$ states at time $t$; the components   of the value function $V^t\in\R^S$ represent the expected cost for an agent in the corresponding state at time $t$. The family of  vectors $\{(\pi^t,V^t)\}_{t=0}^T$ must satisfy certain optimality conditions we are going to explain. Given an initial distribution of the agents $\bar\pi^0$, the evolution of the population follows the law 
\begin{equation*}
		\pi^{t+1}(i)=\sum_{j=1}^{S}P^{t}_{ji}\pi^{t}(j),\qquad t=0,\dots,T-1,
\end{equation*}
where  the  transition matrix $P^t$  is a $S\times S$ stochastic matrix with $P^t_{ij}$ representing the probability that an agent  at the state $i$ moves in the state $j$. Given the terminal cost $\bar V^T\in\R^S$, the matrix $P^t$, $t=0,\dots,T-1$, is obtained by minimizing the   cost functional
 \[
	J^t(i,P)=		 \sum_{j=1}^{S} P_{i,j} \left( c (P_{ij})+ \e\log(P_{ij})+ F(i,\theta)+ V^{t+1}(j)\right)
\]
with   the value function defined by $V^t(i)=\inf_{P}[ J^t(i,P)]$, $i,\dots,S$. The pairs $\{(\pi^t,V^t)\}_{t=0}^T$ satisfy 
\[\left\{
		\begin{array}{ll}
			V^t(i)=\displaystyle{\min_{P_i:\,P_{ij}\ge 0, \sum_j P_{ij}=1}}\left\{\sum_{j=1}^S P_{ij} \left(c ( P_{ij} )+\e\log(P_{ij})+F(i,\theta)+V^{t+1}(j)\right)\right\},\\ [8pt]
			\pi^{t+1}(i)= \sum_{ j=1}^S   P^t_{ji}\pi^t(j),\\ [8pt]
			V^T=\bar V^T,\,\pi^0=\bar\pi^0,\, 
		\end{array}
		\right.
\]
for $t=0,\dots, T-1$, where  the transition matrix $P^t$ in the Fokker-Planck equation is composed of rows $P^t_i=\{P^t_{ij}\}_{j=1}^S$ which realize  the minimum in the Hamilton-Jacobi-Bellman equation. The previous system is given by a backward Hamilton-Jacobi-Bellman equation for the value function and  a forward Fokker-Planck equation for the distribution of agents with the corresponding final and initial condition.
\\
The    stationary MFG system, see \eqref{eq:intro_MFG} for $K=1$, is obtained by considering  the long-run average cost
\[
J(i,\{P^t\}_{t\in\N})={\lim\sup}_{T\to\infty}\frac{1}{T}\sum_{t=0}^{T}\sum_{j=1}^{S} P^t_{ij} \left( c (P^t_{ij})+ \e\log(P^t_{ij})+ F(i,\theta)\right).
\]
In this case, the ergodic cost $\l$ is  given by the infimum of   $J(i,\{P^t\}_{t\in\N})$ and it is independent of the state $i$, while $\pi$ represents the invariant distribution of the process. Note that the solution of the stationary Hamilton-Jacobi-Bellman equation  is defined up to an additive constant and therefore the normalization condition $\sum_{j=1}^S V(j)=0$  is added.\\
In the multi-population system \eqref{eq:intro_MFG}, each population satisfies a control problem as before, and the interaction among the different populations is given by the value $\th_k$ and the mixing coefficients $\a_k$.
%}

	%%%%%%%%%%%%%%%%%%%%%%%%%%%%%%%%%%%%%%%%%%%%%%%%%%%%%%%%%%%
	% A Mean Field Games approach to Bernoulli mixture models %
	%%%%%%%%%%%%%%%%%%%%%%%%%%%%%%%%%%%%%%%%%%%%%%%%%%%%%%%%%%%
	\section{A Mean Field Games approach to    Bernoulli mixture models}\label{sec:bernoulli}
	In this section, we describe a Mean Field Games approach for a mixture of multivariate  Bernoulli distributions. This method is used to cluster high-dimensional binary data. 
	In order to explain the technique, we start  describing     an application of the previous model   to the classification of   handwritten digits  (see \cite[Cap. 9]{bishop}). 
	Consider a database  $\mX= \{x_1, \dots, x_N\}$ of   images representing    hand-written digits of $K$ numbers  between 0 and 9. Each image, which is originally given by a square of $\sqrt{D}\times \sqrt{D}$  pixels  in grey scale, is turned in a binary   vector $x=(x^{1}, \dots, x^{D})$,  of size $D$, by setting the elements whose value exceeds  $1/2$ to state $1$ and the remaining to state $0$, with $1$ corresponding to a white pixel, $0$ to a black one. 
	Assuming to know $K$, the aim   is to subdivide the data set $\mX$ in $K$ clusters, where each cluster  is represented by the component $\pi_k$ of a mixture of Bernoulli distributions and an image is attributed to the cluster which maximizes the responsibility. In a second phase, these clusters can be used to identify the digit corresponding to a new image, but we will not consider this problem here. In Figure \ref{fig:bern_example_digits}, we see some samples of digits taken from database MNIST \cite{url_MINST}.
	\begin{figure}[!h]
		\centering
		\includegraphics[width =0.9\textwidth]{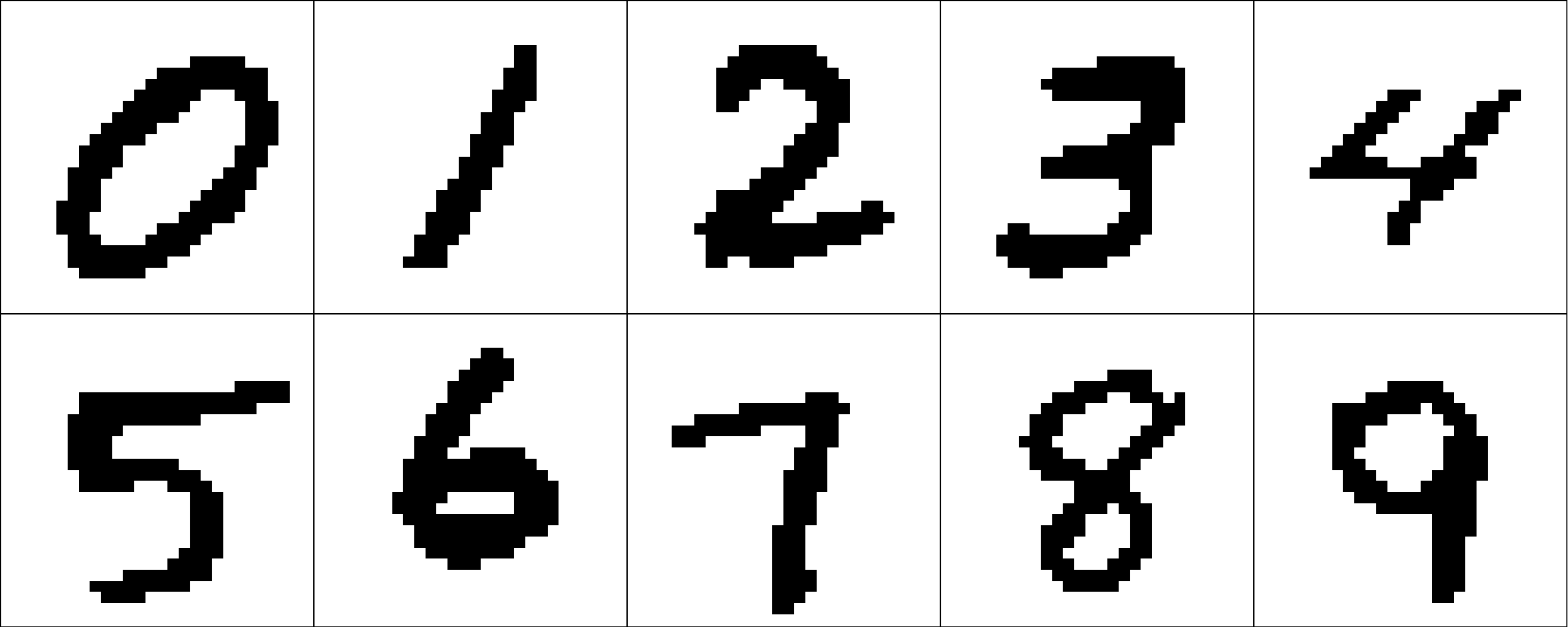}
		\caption{Samples of hand-written digits from the MNIST database}
		\label{fig:bern_example_digits}
	\end{figure}
	We suppose that the digits are  the i.i.d. observations  of a random variable $X$ in $\R^D$. Assuming that the colour of each  pixel is independent of all the other  ones, this implies that the $D$ components of the random vector $X=(X^1,\dots,X^D)$, with $X^d\in\{0,1\}$, are mutually independent. Hence we can describe the distribution of $X$ as a mixture of $K$  multivariate Bernoulli distributions
	\begin{equation}\label{eq:pi}
		\pi(x)= \sum_{k=1}^{K} \alpha_k \pi_k(x ),
	\end{equation}
	where each measure $\pi_k$, defined by
	\begin{equation}\label{eq:pi_k}
		\begin{split}
			&\pi_k(x)=\prod_{d=1}^{D}\pi_k^d(x^d),\quad x=(x^1,\dots,x^D)\in\{0,1\}^D\\
			&\pi_k^d(x^d)=(\mu_k^d)^{x^{d}}(1- \mu_k^d)^{1-x^{d}},\quad x^d\in\{0,1\}
		\end{split}
	\end{equation} 
	represents  the distribution of a specific digit in the $k^{th}$ cluster.
	The aim is to design a $K$-populations MFG system to find the unknown parameters $\alpha= (\alpha_1, \dots, \alpha_K)\in\R^K$ and $\mu=(\mu_1,\dots,\mu_K)\in\R^{K\times D}$ with $\mu_k=(\mu_k^1,\dots,\mu_k^D)$, in such a way that the measure $\pi$ in \eqref{eq:pi}  optimally fits   the data set $\mX$. Following the analysis  described in Section \ref{sec:mixture_models}, we introduce a $K$-dimensional latent random variable $Y$ that specifies which element of the mixture \eqref{eq:pi} has generated a given data point $x_n$ and,
	for $x_n\in\mX$, we define the responsibilities 
	\begin{equation*}
		\g_k(x_n)=%\PP[Y_k=1|X=x_n]=
		\frac{\a_k \pi_k(x_n)}{\pi(x_n)}\qquad k=1,\dots,K.
	\end{equation*}  
	Note that $\g_k(x_n)$ are defined in terms of the unknown  measure $\pi$. 
	Given the responsibilities, we define 
	the average value of the data set  $\theta_k=(\theta_k^1,\dots,\theta_k^D)\in\R^D$ with respect to  the $k^{th}$ component of the mixture by
	\begin{align}
		&\theta_k^d=
		%\EE[X^d| Y_k=1]=
		\dfrac{\sum_{n=1}^{N}\g_k(x_n)x_{n}^d}{\sum_{n=1}^{N}\g_k(x_n)}, \quad d=1,\dots,D.\label{eq:bern_average}
	\end{align}
	By definition, $\theta_k^d\in [0,1]$.
	%%%%%%%%%%%%%%%%%%%%%%
	For each $d=1,\dots, D$, we consider the following  2-states $K$-populations MFG system.
	\begin{equation}\label{eq:bern_MFG}
		\left\{
		\begin{array}{ll}
			V_k^d(0)= \min_{p \in [0,1]}\{p (-\frac{1-p}{2}+\e \log (p) + V_k^d(0)) \\[4pt]
			\qquad+(1-p )(-\frac{p}{2}+\e \log(1-p )+ V_k^d(1))\}-\lambda_k^d +  (\theta_k^d)^2 \\[8pt]
			V_k^d(1)= \min_{q \in [0,1]} \{(1-q )(-\frac{q}{2}+\e \log (1-q ) +V_k^d(0))\\[4pt]
			\qquad + q (-\frac{1-q}{2}+\e \log(q ) +V_k^d(1))\}- \lambda_k^d +  (1- \theta_k^d)^2 \\ [8pt]
			\pi_k^d(0)= p \pi_k^d(0)+(1-q )\pi_k^d(1)\\[4pt]
			\pi_k^d(1)= (1-p )\pi_k^d(0)+q \pi_k^d(1)\\[8pt]
			\pi_k^d\ge 0,\,\sum_{x\in \{0,1\}}\pi_k^d(x)=1, \,\sum_{x\in \{0,1\}}V_k^d(x)=0.
		\end{array}
		\right.
	\end{equation}
	where the values $(p,q)$ in the second couple of equations are given by controls realizing the minimum in the first couple of equations.
	The family of  MFG systems \eqref{eq:bern_MFG} is completed with the global  coupling condition
	\begin{equation}\label{eq:bern_weights}
		\alpha_k=\frac{1}{N}\sum_{n=1}^N\gamma_k(x_n).
	\end{equation}
	For each $d=1,\dots,D$ fixed, the unknowns in the  $K$-populations MFG system \eqref{eq:bern_MFG}   are
	the couple $(V_k^d,\l_k^d)$, with $V_k^d\in\R^2$ and $\l_k^d\in\R$, solving the Hamilton-Jacobi-Bellman equation, and the binomial distributions  $\pi_k^d$,    solving the Fokker-Planck equation, while the  weights  $\a_k\in [0,1]$ are global unknowns, independent of the index $d$. 
	It is important to observe that also  the transition matrix
	\begin{equation}\label{eq:bern_trans_matrix}
		P=\begin{bmatrix}
			p&1-p\\1-q& q
		\end{bmatrix},
	\end{equation}
	in the Fokker-Planck equation depends on $k$ and $d$ and, under appropriate assumptions (see Section \ref{sec:categorical}), is  univocally determined by the couple of controls $p,q$ which minimize  the first two Hamilton-Jacobi equations.\\
	For $d$ fixed, the coupling cost $((\theta_k^d)^2, (1-\theta_k^d)^2)$  is defined in such a way that  a population  described by the  density function $\pi_k^d$    distributes  in the two states $\{0,1\}$   so that   its expected value $\m_k^d$ is as close as possible to the average  value $\theta_k^d$ of the data set. Hence this forces the expected value of the $k$-th component $\pi_k$ of the mixture to be close to the average of  $X$,  given the occurrence of the event $Y_k=1$. \\
	Note  that the value $\theta_k^d$ depends on the responsibilities $\g_k$, see \eqref{eq:bern_average}, and therefore, in turn, on the complete distribution $\pi$. Hence, the systems \eqref{eq:bern_MFG} are coupled also with respect to the index $d$   by means of the quantities $\theta_k^d$. Indeed, the covariance matrix of   \eqref{eq:pi}  
	is not diagonal (see \cite[Equation (9.50)]{bishop}) and the model can capture correlation among  the components of the mixture. The well-posedness of problem \eqref{eq:bern_MFG}-\eqref{eq:bern_weights} will be discussed in Section \ref{sec:categorical} in the more general setting of the categorical distributions.\\
	%%%%%%%%%%%%%%%%%%%%%%%%%%%%%%%%%%%%%%%%%%%%%%%%%%%%%%%%%%%%%%%%%%%%%%%%%%%%%%
	We now show the connection between the multi-population MFG system \eqref{eq:bern_MFG}  and the  functional \eqref{eq: expected_likelihood}, which, in the case of a multivariate Bernoulli 
	distribution   \eqref{eq:pi_k}, is plainly given by  
	\begin{equation}\label{eq:bern_log_lik}
		\begin{split}
			\mathcal{\tilde{L}}(\alpha,\mu; \mX)= \sum_{n=1}^{N}\sum_{k=1}^{K} & \gamma_k(x_n)\bigg(\log(\alpha_k)+ \sum_{d=1}^{D}(x_n^d \log (\mu_k^d)\\
			&+ (1- x_n^d)\log (1- \mu_k^d))\bigg)
		\end{split}
	\end{equation}
	for $\a=(\a_1,\dots,\a_k)$, $\mu=(\mu_1,\dots,\mu_K)\in\R^{K\times D}$, where $\mu_k=(\mu_k^1,\dots,\mu_k^D)$, $k=1,\dots,K$.
	Note that, since the log-likelihood has to be maximized, we are interested in the critical values of $\mathcal{\tilde{L}}(\alpha,\mu; \mX)$  for the values of the parameters $\a_k$, $\mu_k^d$ inside the interval $(0,1)$. By \eqref{eq:bern_log_lik}, the necessary conditions \eqref{eq:mix_critical1}-\eqref{eq:mix_critical2}
	can be computed explicitly in terms of responsibilities   and they are given by
	\begin{equation}\label{eq:bern_consistency}
		\begin{cases}
			\displaystyle	 \mu_k^d= \dfrac{ \sum_{n=1}^{N} \gamma_k(x_n)x_{n}^d}{ \sum_{n=1}^{N}\gamma_k(x_n)} \\[12 pt] 
			\alpha_k=\dfrac{1}{N}\sum_{n=1}^{N}\gamma_k(x_n).
		\end{cases}
	\end{equation}
	In the next proposition, we show that  a Bernoulli mixture satisfying \eqref{eq:bern_consistency} can be always   obtained as a solution of the MFG system \eqref{eq:bern_MFG} with $\e=0$. Moreover, a solution of problem \eqref{eq:bern_MFG}-\eqref{eq:bern_weights} with $\e=0$ gives a mixture satisfying \eqref{eq:bern_consistency}.
	%%%%
	\begin{proposition}\label{p:cons_MFG}
		We  have 
		\begin{itemize}
			\item[(i)]
			Let $\pi(x)=\sum_{k=1}^{K} \alpha_k \pi_k(x )$, with $\pi_k$ as in \eqref{eq:pi_k}, be a Bernoulli mixture   satisfying \eqref{eq:bern_consistency}.  Then,   the family of quadruples  $(V_k, \lambda_k, \pi_k, \a_k)$, $k=1,\dots,K$, with 
			\begin{equation}\label{eq:bern_sol}
				\begin{array}{l}
					V_k=(V_k^1,\dots,V_k^D), \quad V_k^d(x)=(\frac{1-2\mu_k^d}{2})^x(\frac{2\mu_k^d-1}{2})^{1-x}  , \,x\in \{0,1\},\\[8pt]
					\l_k=(\l_k^1,\dots,\l_k^D),\quad\lambda_k^d=0,\\[8pt]
					\pi_k=(\pi_k^1,\dots,\pi_k^D),\quad \,\pi^d_k(x)=(\mu_k^d)^{x}(1- \mu_k^d)^{1-x},\,x\in \{0,1\},
				\end{array}
			\end{equation}
			is a  solution of \eqref{eq:bern_MFG}-\eqref{eq:bern_weights} with $\e=0$ and $\theta_k^d$ as in \eqref{eq:bern_average}.
			\item[(ii)] Let $\{(V_k,\l_k,\pi_k,\a_k)\}_{k=1}^K$ be a solution of the  MFG system \eqref{eq:bern_MFG}-\eqref{eq:bern_weights} with $\e=0$ and $\theta_k^d$ as in \eqref{eq:bern_average}. Then $\pi(x)= \sum_{k=1}^{K}\alpha_k \pi_k(x)$ is a Bernoulli mixture verifying \eqref{eq:bern_consistency}.
		\end{itemize}
	\end{proposition}
	\begin{proof}
		\textit{(i)}
		Consider a  mixture $\pi$ defined as in \eqref{eq:pi}-\eqref{eq:pi_k} and  satisfying \eqref{eq:bern_consistency}.
		Then, each component $\pi_k^d$ satisfies the 2-states  Fokker-Plank equation 
		\begin{equation}\label{eq:FP}
			\left\{
			\begin{array}{ll}
				\pi_k^d(0)= p \pi_k^d(0)+(1-q )\pi_k^d(1)\\[4pt]
				\pi_k^d(1)= (1-p )\pi_k^d(0)+q \pi_k^d(1)\\[8pt]
				\pi_k^d\ge 0, \ \sum_{x\in \{0,1\}}\pi_k^d(x)=1
			\end{array}
			\right.
		\end{equation}
		where the transition matrix $P$ is  defined as in \eqref{eq:bern_trans_matrix}  with $p=1-\mu_k^d$, $q=\mu_k^d$. Moreover, we observe that by condition \eqref{eq:bern_consistency}, the expected values $\th_k^d$ in \eqref{eq:bern_MFG}, defined in \eqref{eq:bern_average}, coincide with $\m_k^d$. Therefore the couple
		$(V_k^d(x),\lambda_k^d)=((\frac{1-2\mu_k^d}{2})^x(\frac{2\mu_k^d-1}{2})^{1-x},0), \ x \in [0,1]$ satisfies the first and second equations of system \eqref{eq:bern_MFG} with $\e=0$ and with optimal controls given by $p^*=1-\mu_k^d$ and $q^*=\mu_k^d$. 
		\\In addition, the homogeneous condition $\sum_{x \in [0,1]}V_k^d(x)=0$ is fulfilled.
		\\Hence, we conclude that the transition matrix $P$ in \eqref{eq:FP} is given by the optimal controls $p^*$ and $q^*$  and therefore \eqref{eq:bern_sol} is a solution of \eqref{eq:bern_MFG} with $\e=0$.
		\par
		%%%
		\textit{(ii)} Let $\{(V_k,\l_k,\pi_k,\a_k)\}_{k=1}^K$ be a solution of the  MFG system \eqref{eq:bern_MFG}-\eqref{eq:bern_weights} with $\e=0$. Then, since $(V_k^d, \lambda_k^d)$ is the     solution to Hamilton-Jacobi-Bellman equation %\eqref{eq:HJB},
		\[
			\left\{
			\begin{array}{ll}
				V_k^d(0)= \min_{p \in [0,1]}\{p (-\frac{1-p}{2}+ V_k^d(0)) \\[4pt]
				\qquad+(1-p )(-\frac{p}{2}+ V_k^d(1))\}-\lambda_k^d +  (\theta_k^d)^2 \\[8pt]
				V_k^d(1)= \min_{q \in [0,1]} \{(1-q )(-\frac{q}{2}+  V_k^d(0))\\[4pt]
				\qquad + q (-\frac{1-q}{2}+ V_k^d(1))\}- \lambda_k^d +  (1- \theta_k^d)^2,\\[4pt]
				\sum_{x\in \{0,1\}}V_k^d(x)=0,
			\end{array}
			\right.
		\]
	we have   $(V_k^d, \lambda_k^d)= \bigg( (\frac{1-2\theta_d^k}{2})^x(\frac{2\theta_d^k-1}{2})^{1-x}, 0\bigg)$ with    the optimal controls $p$ and $q$  given by    
		\begin{equation}\label{eq:bern_optimal}
			p=1-\theta_k^d,\quad  q=\theta_k^d .
		\end{equation} 
		It follows that, given the transition matrix $P$ as in \eqref{eq:bern_trans_matrix} with $p$,$q$ as in \eqref{eq:bern_optimal},   the solution to the Fokker-Planck equation \eqref{eq:FP} is   a Bernoulli distribution $\pi_k^d(x)$ of parameter   $\mu_k^d= \theta_k^d$. Therefore, by \eqref{eq:bern_average},
		$\mu_k^d$ satisfies the first condition in \eqref{eq:bern_consistency}. Moreover, since the coefficients $\a_k$ are given by \eqref{eq:bern_weights},  also the second condition in \eqref{eq:bern_consistency} is satisfied. 
	\end{proof}
	%%%
	\begin{remark}
	If $\e>0$, then it is possible to show  that the MFG system \eqref{eq:bern_MFG}  characterizes the critical points of the   functional 
		\begin{align*}
			\mathcal{\tilde{L}}_\e(\alpha,\mu; \mX)= \sum_{n=1}^{N}\sum_{k=1}^{K} & \gamma_k(x_n)\bigg(\log(\alpha_k)+ \sum_{d=1}^{D}x_n^d \log (f_\e(\mu_k^d))\\
			&+ (1- x_n^d)\log (f_\e(1- \mu_k^d))\bigg)
		\end{align*}
		where
		\[f_{\e}(\mu)=\mu+\frac{\e}{2}\log\left(\frac{\mu}{1-\mu}\right).\]
		Note that $f_\e(\mu)\in (0,1)$ for  $\mu\in (\d_\e,1-\d_{\e})$, for some appropriate constant $\d_\e\in(0,1)$ with $\lim_{\e\to 0^+}\d_\e=0$. Hence the functional  $\mathcal{\tilde{L}}_\e$ is defined for $\mu\in (\d_\e,1-\d_{\e})$ and, for $\e>0$, the MFG system \eqref{eq:bern_MFG} gives non degenerate multinomial  Bernoulli distributions.
	\end{remark}
	%%%%%%%%%%%%%%%%%%%%%%%%%%%%%%%%%%%%%%%%%%%%%%%%%%%%%%%%%%%%%%%%%%%%%%%%%%%%%%%
	%  A Mean Field Game approach to mixture models of categorical distributions  %
	%%%%%%%%%%%%%%%%%%%%%%%%%%%%%%%%%%%%%%%%%%%%%%%%%%%%%%%%%%%%%%%%%%%%%%%%%%%%%%%
	\section{A Mean Field Games approach   to mixture models of categorical distributions}\label{sec:categorical}
	In the  classification of handwritten digits previously described, the components of the random variable $X$, which generates the data set $\mX$, can take only one of two possible values, i.e. $0$ and $1$ corresponding to  a white or a black pixel. 
	In other models, see for example  the Fashion-MNIST dataset \cite{url_FashionMINST},   a  discrete random variable  can take only one of a certain number $S$ of mutually exclusive  states. In this case, $X$ is called a categorical random variable. In this section, we introduce a MFG approach to the optimization of  the parameters for a mixture of   categorical distributions, i.e. discrete probability measures associated with a categorical random variable. \\
	We denote by $\cS=\{1,\dots,S\}$ the state space of a component $X^d$ of $X$.
	A probability measure on the state space $\cS$ can be identified with a vector $p\in \mS$, where  
	$\mS=\{p=(p(1),\dots,p(S)):\, p(i)\ge 0,\,\, \sum_{i=1}^Sp(i)=1\}$
	is the probability simplex.  The set of the  $S\times S$ stochastic matrices is identified with $\maS$.\\
	We assume that the   data set $\mX$ is generated  by   the i.i.d. observations  of a random variable $X$ in $\R^D$, whose    components $X^d$    are categorical random variables with values in $\cS$ and independent from each other. In order to find a   representation of the distribution of $X$, we consider the  mixture 
	\begin{equation}\label{eq:cat_pi}
		\pi(x)= \sum_{k=1}^{K} \alpha_k \pi_k(x ),\quad  x=(x^1,\dots,x^D)\in \cS^D,
	\end{equation}
	where each measure $\pi_k$ is given by
	\begin{equation}\label{eq:cat_pi_k}
		\pi_k(x)=\prod_{d=1}^{D}\pi_k^d(x^d),\, 
	\end{equation} 
	with $\pi_k^d \in\mS$, $d=1,\dots,D$,   a  categorical distribution.   Note that \eqref{eq:cat_pi_k} is consequence of  the assumption that the components of $X$ are mutually independent. We denote with $\cP$ the space of the multinomial categorical measures, i.e. the space of the  measures defined as in \eqref{eq:cat_pi_k}. We also identify a measure $\pi_k$ on $\cS^D$ with a vector $\pi_k=(\pi_k^1,\dots,\pi_k^D)$ such that  $\pi_k^d\in\mS$. 
	\begin{remark}
		The Fashion-MNIST dataset  is composed of grey-scale, $28\times 28$ pixels images of $10$ different types of fashion products. In this case we have  $D=784$, corresponding to the pixels of the images, $S=256$, corresponding to the grey-scale levels of a single pixel, and $K=10$, corresponding to the different fashion objects in the dataset. 
	\end{remark}
	\noindent
	Given the data set $\mX=\{x_n\}$, for $k=1,\dots,K$ we define  the responsibilities  by
	\begin{equation}\label{eq:cat_resp}
		\g_k(x_n)=\PP[Y_k=1|X=x_n]=\frac{\a_k \pi_k(x_n)}{\pi(x_n)}\qquad \,n=1,\dots,N.
	\end{equation} 
	and the vector $\theta_k=(\theta_k^1,\dots,\theta_k^D)\in \mS^D$,   $\theta_k^d= (\theta_k^d(1), \dots, \theta_k^d(S)) \in \mS$, by
	\begin{equation}\label{eq:cat_average}
		\theta_k^d(i)=\dfrac{\sum_{n=1}^{N}\g_k(x_n)[x_{n}^d=i]}{\sum_{n=1}^{N}\g_k(x_n)}  \quad i=1,\dots,S
	\end{equation}
(here the Iverson brackets $[x_n^d=i]$ evaluate $1$ if the realization of the random variable $X^d$ corresponding to the $n^{th}$ entry of the data set $\mathcal{X}$ assumes state $i$,  $0$ otherwise).  The vector $\theta_k$ represents  the weighted probabilities of the random variable  $X^d$ with respect to the $k^{th}$ component of the mixture computed on the data set $\mX$.
	%%
	
	%Note that, by definition, $\th_k^d(i)\in [0,1]$ {\color{red}and, in particular, $\theta_k^d \in \mathbb{S}$. In fact \[\sum_{i=1}^{S}\sum_{n=1}^{N}\gamma_k(x_n)[x_n^d=i]= \sum_{n=1}^{N}\sum_{i=1}^{S}\gamma_k(x_n)[x_n^d=s]= \sum_{n=1}^{N}\gamma_k(x_n) \]
	%since $X^d$ is a categorical random variable with values in $\mathcal{S}$.} 
	For $k=1,\dots,K$, $d=1,\dots,D$ and $i=1,\dots,S$, we consider the multi-population MFG system
	\begin{equation}\label{eq:cat_MFG}
		\left\{
		\begin{array}{ll}
			V_k^d(i)=\displaystyle\min_{P_i: P_{ij}\geq 0, \ \sum_j P_{ij}=1}\left\{\sum_{j=1}^S P_{ij} \left(c (P_{ij})+\e\log(P_{ij})+F(i, \theta_k^d)+V_k^d(j)\right)\right\}-\l_k^d,\\ [8pt]
			\pi_k^d(i)= \sum_{ j=1}^S P^k_{ji}\pi_k^d(j),\\[8pt]
			\pi_k^d(i)\ge 0,\, \sum_{i=1}^S\pi_k^d(i)=1, \,\sum_{i=1}^S V_k^d(i)=0,\\[8pt]
			\alpha_k=\frac{1}{N}\sum_{n=1}^N\g_k(x_n),
		\end{array}
		\right.
	\end{equation}
	where  the transition matrix $P^k$ in the Fokker-Planck equation is composed of rows $P^k_i=\{P^k_{ij}\}_{j=1}^S$ which realize  the minimum in the Hamilton-Jacobi-Bellman equation.  Note that, for the optimal values of $P_i$, $c(P_{ij})$  is a transition cost from the state $i$ to the state $j$ and the entropy	penalization term $\e\log(P_{ij})$ enforces the agent to diversify their choices, so that $P_{ij}>0$ for any $i,j=1,\dots,S$. We assume that   $c\in C^1([0,1])$ with $pc(p)$   convex  for $p\in [0,1]$ and $F:\mathcal{S}\times \mathbb{S}\rightarrow \mathbb{R}$ is such that $F(i, \cdot)$ is bounded and continuous for all $ i  \in \mathcal{S}$.\\
	A solution of \eqref{eq:cat_MFG} is a family of quadruples  $(V_k, \lambda_k, \pi_k, \a_k)$, $k=1,\dots,K$, where 
	\begin{equation*}
		\begin{array}{ll}
			V_k=(V_k^1,\dots,V_k^D), \, &V_k^d =(V_k^d(1),\dots, V_k^d(S) )\in\R^S,\,d=1,\dots,D;\\
			\l_k=(\l_k^1,\dots,\l_k^D),&\lambda_k^d\in \R,\,d=1,\dots,D;\\
			\pi_k=(\pi_k^1,\dots,\pi_k^D),&\pi^d_k=(\pi_k^d(1),\dots, \pi_k^d(S))\in\mS,\,d=1,\dots,D;\\
			\a_k\in [0,1].
		\end{array}
	\end{equation*}
	As observed in the previous section, for fixed $d\in \{1,\dots,D\}$,     \eqref{eq:cat_MFG} gives a $K$-populations MFG system on a   $S$-states space. The system  is globally coupled by means of the   vectors $\theta_k^d$, which depend on the full measure $\pi$ via  the responsibilities $\gamma_k$, and by means  of the weights $\alpha_k$. Note that, a priori, the responsibilities are not well  defined since $\pi$ could vanish for some $x_n\in\mX$. However, we will prove in Theorem \ref{thm:cat_exixts} that, due to the entropy penalization term, there exists a solution of \eqref{eq:cat_MFG} for which $\pi$ cannot vanish on the data set $\mX$ and therefore $\g_k$ and $\th_k^d$ are well defined.
	\begin{theorem} \label{thm:cat_exixts}
		For any $\e>0$, there exists a solution $(V_k, \lambda_k, \pi_k, \alpha_k) $, $k=1,\dots,K$, of \eqref{eq:cat_MFG}.
	\end{theorem}
	\begin{proof}
		We define the
		\begin{equation}\label{eq:cat_FP_space}
			\begin{split}
				\cD=\Big \{ (\a,\pi)=&(\a_1,\dots,\a_K,\pi_1,\dots,\pi_K):\, \a_k\in [0,1],\, \sum_{k=1}^K\a_k=1,  \\
				& \pi_k=\prod_{d=1}^{D}\pi_k^d\in\cP, \,\min_{i\in\cS}\pi_k^d(i)\ge \d>0,\, d=1,\dots,D\Big\},
			\end{split}
		\end{equation} 
		where $\d$ is a constant to be fixed later.   It is easy to see that $\cD$ is a convex  and compact set  with
		respect to the topology  of  $\R^K\times (\mS^D)^K$. We define a map $\Psi$ on $\cD$ in the following way:\\
		%%%
		Given $(\a,\pi)\in\cD$, we set for $k=1,\dots,K$
		\begin{align}
			&\g_k(x_n)=\frac{\a_k\pi_k(x_n)}{\sum_{k=1}^K\a_k\pi_k(x_n)}, \quad n=1, \dots, N\label{eq:cat_FP_resp}\\
			&\theta_k^d(i)=\frac{\sum_{n=1}^{N}\g_k(x_n)[x_{n}^d=i]}{\sum_{n=1}^N\g_k(x_n)}, \quad d=1,\dots,D \quad i =1, \dots S \label{eq:cat_FP_average}
		\end{align}
		and $\theta_k^d=(\theta_k^d(1), \dots, \theta_k^d(S))$.\\
		Note that, since $\min_{i\in\cS}\pi_k^d(i)\ge \d>0$, then $\pi_k(x)\ge \d^D$ for any $x\in\cS^D$, and
		therefore 
		\begin{equation}\label{eq:cat_FP_lower}
			\sum_{k=1}^K\a_k\pi_k(x_n)\ge \d^D,\qquad n=1,\dots,N.
		\end{equation}
		Hence $\g_k$ and   $\theta_k^d$ in \eqref{eq:cat_FP_resp}-\eqref{eq:cat_FP_average} are well defined. 
		For each $d=1,\dots,D$, we consider the $K$-populations $S$-states space MFG system
		\begin{equation}\label{eq:cat_FP_MFG}
			\left\{
			\begin{array}{ll}
				V_k^d(i)=\displaystyle\min_{P_i: P_{ij}\geq 0, \ \sum_j P_{ij}=1}\left\{\sum_{j=1}^S P_{ij}\left(c(P_{ij})+\e\log(P_{ij})+F(i, \theta_k^d)+V_k^d(j)\right)\right\}-\l_k^d\\ [8pt]
				\rho_k^d(i)= \sum_{j=1}^S P^k_{ji}\,\rho_k^d(j)\\[8pt]
				\rho_k^d(i)\ge 0,\, \sum_{i=1}^S\rho_k^d(i)=1, \,\sum_{i=1}^S V_k^d(i)=0
			\end{array}
			\right.
		\end{equation}
		for $i=1,\dots,S$, where  the transition matrix $P^k$ in the Fokker-Planck equation is composed of rows $P^k_i=\{P^k_{ij}\}_{j=1}^S$ which realize  the minimum in the Hamilton-Jacobi-Bellman equation.
		Since the coefficients $\theta_k^d(i)$ are  given (see \eqref{eq:cat_FP_average}), the previous systems are not coupled with respect to the index $d$. By Theorem \ref{app:theorem_positivity}, for any $d=1,\dots,D$, there exists  a unique solution $(V_k^d,\l_k^d,\rho_k^d)$ to \eqref{eq:cat_FP_MFG}
		with 
		\begin{equation}\label{eq:cat_FP_positiv}
			\min_{i\in\cS}\rho_k^d(i)\ge C(\e)>0,\qquad  k=1,\dots,K.
		\end{equation} 
		Set  $\rho=(\rho_1,\dots,\rho_K)$,  where for any $k=1,\dots,K$ the vector $\rho_k=(\rho_k^1,\dots,\rho_k^D)\in \cP$ has components  given by the solutions $\rho^d_k$ of \eqref{eq:cat_FP_MFG}, 
		and 
		\begin{equation*}
			\b_k=\frac{1}{N}\sum_{n=1}^N\g_k(x_n).
		\end{equation*}
		Then, the map $\Psi$ is defined by $\Psi(\a,\pi)=(\b,\rho)$. Fixing the constant $\d$ in \eqref{eq:cat_FP_space} smaller than the constant  $C(\e)$ in \eqref{eq:cat_FP_positiv}, then it follows that $\Psi$ maps the set $\cD$ into itself. We prove that $\Psi$ is a continuous map on $\cD$    with  respect to the topology of $\R^K\times(\mS^D)^K$. Consider a sequence $(\a^{(h)},\pi^{(h)})\in\cD$, $h\in\N$, converging to $(\a,\pi)\in \cD$
		and denote by $(\rho^{(h)},\b^{(h)})$, $(\rho,\b)$ the corresponding images by means of $\Psi$. Given the vector $(\theta_{k}^d)^{(h)}$ whose components are defined by 
		\[(\theta_{k}^d)^{(h)}(i)=\frac{\sum_{n=1}^{N}\g_{k}^{(h)}(x_n)[x_{n}^d=i]}{\sum_{n=1}^N\g_{k}^{(h)}(x_n)}, \quad k=1,\dots,K,\,d=1,\dots,D,\]
		and recalling  \eqref{eq:cat_FP_lower}, it is immediate that $(\theta_{k}^d)^{(h)}(i)\to\theta_k^d(i)$ for $h\to\infty$, and hence $(\theta_{k}^d)^{(h)} \to \theta_k^d$ for $h \to \infty$ for any $k=1,\dots,K$, $d=1,\dots,D$. By   continuity,   $F(i, (\theta_{k}^d)^{(h)})$ converges 
		%\red{uniformly: togliere}
		to $F(i, \theta_k^d)$.
		By Lemma \ref{app:lemma_HJB} and \ref{app:lemma_continuity_Nash}, it follows that the vectors $P^{(h)}_i$, which attain  the minimum in the first equation of \eqref{eq:cat_FP_MFG} for $((V_{k}^d)^{(h)},(\l_{k}^d)^{(h)})$, converge  to the corresponding vector $P_i$ which attains the minimum in the first equation of \eqref{eq:cat_FP_MFG} for $(V_k^d,\l_k^d)$. Since the transition matrices $P^{(h)}$, composed of the rows $P^{(h)}_i$, $i=1,\dots,S$, converge to the transition matrix $P$, composed of the rows $P_i$, $i=1,\dots,S$ , it follows that the corresponding invariant distribution $(\rho_{k}^d)^{(h)}$ converge to $\rho_k^d$, $d=1,\dots,D$ and therefore $(\rho_{k})^{(h)}=\prod_{d=1}^D (\rho_{k}^d)^{(h)}$ converges to $\rho_k=\prod_{d=1}^D \rho_k^d$. Moreover, recalling \eqref{eq:cat_FP_positiv}, it follows that $\b_{k}^{(h)}$ converges to $\b_k$. Hence the map $\Psi$ is continuous.\\
		By the Brouwer's fixed point Theorem, there exists $(\a,\pi)\in\cD$ such that $\Psi(\a,\pi)=(\a,\pi)$ and therefore
		a solution to \eqref{eq:cat_MFG}.
	\end{proof}
	Note that, in this framework, it is, in general, not reasonable  to obtain uniqueness for the system \eqref{eq:cat_MFG}, since  the  log-likelihood functional  can not be concave. This also corresponds to the fact that, in cluster analysis, there may be several admissible partitions of a given data set.
	As for the case $S=2$, see Proposition \ref{p:cons_MFG}, it is possible to show that, identifying the state $i\in \cS$   with the vertex  of coordinate 
	$$T_i=(0,\dots,\underset{i}{1},\dots,0)$$
	of the   simplex in $\R^S$, taking $c(p)=-(1-p)/2$ and $F(i,\th)$ equal to the square of   distance of $\th$ from $T_i$, then the solutions of the MFG system \eqref{eq:cat_MFG} for $\e=0$ 
	are associated with the critical points of the  the log-likelihood functional for a mixture of categorical distributions.
	%%%%%%%%%%%
	\begin{remark}
		Theorem \ref{thm:cat_exixts}	 provides an existence result to the MFG system \eqref{eq:cat_MFG} only  for $\e>0$.
		For $\e=0$, estimate \eqref{eq:cat_FP_positiv} no longer holds and  \eqref{eq:cat_FP_lower} no longer applies. Therefore, it is not guaranteed that the responsibilities $\g_k$ in \eqref{eq:cat_resp} are well defined. 
		On the other hand, for $\e \to 0$, it is possible to prove that, up to a subsequence, the vector $\th_k^d$, defined in \eqref{eq:cat_average} converges to a vector $\bar\th_k^d \in \mS$ and, moreover, the  solution  of  \eqref{eq:cat_MFG} converges to a solution of \eqref{eq:cat_MFG} with $\e=0$ and cost $F(i,\bar\th_k^d)$. The point is  that we cannot   characterize $\bar\th_k^d$ by a formula equivalent to \eqref{eq:cat_average} since again we cannot exclude that the  responsibilities $\g_k$,  corresponding to the solution of the limit system, vanish for some value $x_n\in\mX$.
	\end{remark}
	%%%%%%%%%%%
	\begin{remark}
		For the simplicity of the notation, we assumed that all the components $X^d$ of the random variable $X$, which generates the data set,  are $S$-dimensional categorical random variables. Nevertheless, it is possible to assume that they have   state spaces of different dimension $S_d$, for each $d=1,\dots,D$. In this case, also  the MFG systems \eqref{eq:cat_MFG} are defined on   states spaces of dimension $S_d$ and the results of this section can be easily reformulated in this more general framework. 
	\end{remark}

	%%%%%%%%%%%%%%%%%%%%%%%%%%%%%%%%%%%%%%%%%%%%%%%%%%%%
	%     Numerical approximation and examples         %
	%%%%%%%%%%%%%%%%%%%%%%%%%%%%%%%%%%%%%%%%%%%%%%%%%%%%
	\section{Numerical approximation and examples}\label{sec:numerics}
	In this section we present the main idea for the numerical solution of the MFG system \eqref{eq:cat_MFG}, then we apply the resulting algorithm to some classical tests in cluster analysis. For the reader's convenience, we recall the MFG system here: 
	for $k=1,\dots,K$, $d=1,\dots,D$ and $i=1,\dots,S$, 
	$$
	\left\{
	\begin{array}{ll}
	V_k^d(i)=\displaystyle{\min_{P_i: \ P_{ij}\ge 0, \sum_j P_{ij}=1}}\left\{\sum_{j=1}^S P_{ij} \left(c (P_{ij})+\e\log(P_{ij})+F(i, \theta_k^d)+V_k^d(j)\right)\right\}-\l_k^d,\\ [8pt]
	\pi_k^d(i)= \sum_{ j=1}^S P^k_{ji}\pi_k^d(j),\\[8pt]
	\pi_k^d(i)\ge 0,\, \sum_{i=1}^S\pi_k^d(i)=1, \,\sum_{i=1}^S V_k^d(i)=0,\\[8pt]
	\alpha_k=\frac{1}{N}\sum_{n=1}^N\g_k(x_n),
	\end{array}
	\right.
	$$
	As discussed in Section \ref{sec:categorical}, this system depends on the dimension $S$ of the state space, the number $K$ of populations, and also on the number $D$ of components of the categorical random variable associated to the mixture. This yields a possibly huge non linear problem of size $S\times K\times D$, which is fully coupled, since the parameters $\theta_k$,  $k=1,\dots,K$,     appearing in the cost $F$ depend, through the responsibilities $\gamma_k$, on the whole mixture $\pi$. Moreover, each Hamilton-Jacobi-Bellman equation in the system includes an optimization problem for the optimal controls in the transition matrix $P$.
	
	Following the strategy of the classical EM algorithm, we can mitigate the computational efforts by reducing the above system to $K\times D$ independent sub-systems of size $S$. More precisely, we devise the following iterative procedure, in which each iteration is split into two steps: starting from an arbitrary guess  $\a^{(0)}_k$, $\pi^{(0)}_k$ for the components of the mixture \eqref{eq:cat_pi}, iterate on $h\ge 0$: \\
	
	\noindent E-step: for $k=1,\dots,K$ compute the new responsibilities and weights
	\[\g_k^{(h+1)}(x_n)=\frac{\a^{(h)}_k\pi^{(h)}_k(x_n)}{\sum_{k=1}^K\a^{(h)}_k\pi^{(h)}_k(x_n)}\,,\qquad \alpha_k^{(h+1)}=\frac{1}{N}\sum_{n=1}^N\g^{(h+1)}_k(x_n),\]
	and the  parameters $\th_k$ of the components of the mixture
	\[(\theta_k^d)^{(h+1)}(i)=\frac{\sum_{n=1}^{N}\g^{(h+1)}_k(x_n)[x_{n}^d=i]}{\sum_{n=1}^N\g^{(h+1)}_k(x_n)}, \quad d=1,\dots,D, \quad i =1, \dots S. \]
	\noindent M-step: for $k=1,\dots,K$, $d=1,\dots,D$ solve the $S$-state space MFG sub-system
	\[
	\left\{
	\begin{array}{ll}
	V_k^d(i)=\displaystyle{\min_{P_i:\,P_{ij}\ge 0, \sum_j P_{ij}=1}}\left\{\sum_{j=1}^S P_{ij} \left(c (P_{ij})+\e\log(P_{ij})+F(i, (\theta_k^d)^{(h+1)})+V_k^d(j)\right)\right\}-\l_k^d,\\ [8pt]
	\pi_k^d(i)= \sum_{ j=1}^S P^k_{ji}\pi_k^d(j),\\[8pt]
	\pi_k^d(i)\ge 0,\, \sum_{i=1}^S\pi_k^d(i)=1, \,\sum_{i=1}^S V_k^d(i)=0,
	\end{array}
	\right.
	\]
	to obtain the new mixture components $\pi^{(h+1)}_k$.\\
	
	For a fixed tolerance $\tau>0$, convergence can be checked by evaluating in a suitable norm the condition $\|\theta^{(h+1)}-\theta^{(h)}\|<\tau$. For instance one can reinterpret $\theta^{(h)}$ as a vector in $\R^{S\times K \times D}$ and simply take the Euclidean norm.  
	
	Note that in this EM-like formulation, the coupling is all embedded in the E-step, whereas the M-step can be completely parallelized. Moreover, each MFG sub-system indexed by $(k,d)$ is also decoupled, since the dependency of $(\theta_k^d)^{(h+1)}$ on $\pi$ is frozen at the previous iteration. Hence, the building block of the algorithm is just to solve the Hamilton-Jacobi-Bellman and the Fokker-Planck equations separately (we remove the indices $k, d, h$ to simplify the notation):
	\[
	\left\{
	\begin{array}{ll}
	V(i)=\displaystyle{\min_{P_i:\,P_{ij}\ge 0, \sum_j P_{ij}=1}}\left\{\sum_{j=1}^S P_{ij} \left(c (P_{ij})+\e\log(P_{ij})+F(i,\theta(i))+V(j)\right)\right\}-\l,\\ [8pt]
	\sum_{i=1}^S V(i)=0,
	\end{array}
	\right.
	\]
	
	\[
	\left\{
	\begin{array}{ll}
	\pi(i)= \sum_{ j=1}^S P_{ji}\pi(j),\\[8pt]
	\pi(i)\ge 0,\, \sum_{i=1}^S\pi(i)=1.
	\end{array}
	\right.
	\]
where  the transition matrix $P$ in the Fokker-Planck equation is composed of rows $P_i=\{P_{ij}\}_{j=1}^S$ which realize  the minimum in the HJB equation.	For the Hamilton-Jacobi-Bellman equation, we employ a standard policy iteration algorithm. More precisely, starting from a guess for the optimal transition matrix $P^{(0)}$, we introduce an inner iteration $m\ge 0$ by taking $P^{(m)}$ as minimizer:
	\[
	\left\{
	\begin{array}{ll}
	V^{(m)}(i)=\sum_{j=1}^S P_{ij}^{(m)} \left(c (P_{ij}^{(m)})+\e\log(P_{ij}^{(m)})+F(i,\theta(i))+V^{(m)}(j)\right)-\l^{(m)},\\ [8pt]
	\sum_{i=1}^S V^{(m)}(i)=0.
	\end{array}
	\right.
	\]
	This results in a very simple linear system of size $S+1$ for the unknowns $(V^{(m)}(1),\dots,V^{(m)}(S),\l^{(m)})$, whose solution is then plugged back in the optimization problem for $P$ to get $P^{(m+1)}$:
	\[
	P^{(m+1)}=\arg\min_{P_i: P_{ij} \geq 0, \ \sum_{j}P_{ij}=1}\left\{\sum_{j=1}^S P_{ij} \left(c (P_{ij})+\e\log(P_{ij})+F(i,\theta(i))+V^{(m)}(j)\right)\right\}\,.
	\]
	Iterations on $m$ are performed up to convergence  $\|P^{(m+1)}-P^{(m)}\|<\tau$.\\
	Note that, under the assumptions made on the transition cost, each optimization problem for $P$ is convex with linear constraints $\sum_{i=1}^S P_{ij}=1$, and it can be readily solved with classical algorithms.
	
	Finally, the optimal transition matrix $P$ is plugged in the Fokker-Plank equation, yielding again a simple linear system with linear constraints.
	
	We have to remark that, at present, the proposed algorithm is just heuristic, supported by numerical evidence only, and a complete proof for its convergence is still under investigation. While convergence of the policy iteration method for Hamilton-Jacobi-Bellman equations is a quite well established subject in the literature, at least in the continuous setting (see \cite{bellman,fleming,howard,puterman,puterman1}), a policy iteration method for MFG systems seems new and not much explored yet (see \cite{ccg20}). Here, further difficulties appear in the additional coupling with the expectation step of the algorithm. Nevertheless, we stress that, at this early stage, the aim of the present work is to provide a new perspective to cluster analysis, more than introduce a new method computationally competitive with classical algorithms. We believe that a better understanding of the proposed MFG formulation of the problem, including the case of more general costs, and an analysis on how these choices affect the resulting clusterization, could reveal additional features of the data set, exploiting its structure in a deeper way.
	
	We now consider some classical examples in cluster analysis, in order to show that the MFG approach produces reasonable results. It is important to remark that the solution of the full $K$-populations MFG system is in general not unique, hence the numerical solution obtained with the proposed algorithm depends on the choice of the initial guess $\a^{(0)}_k$, $\pi^{(0)}_k$, for $k=1,\dots,K$. In the following experiments we always choose $\a^{(0)}_k=1/K$, while  $\pi^{(0)}_k$ is built using random numbers in $(0,1)$. Moreover, we choose the costs $c(p)=-(1-p)/2$ and $F(i,\th)=|\th-T_i|^2$, identifying the state $i\in \cS$   with the vertex  of coordinate $T_i=(0,\dots,1,\dots,0)$
	of the   simplex in $\R^S$.  We also set $\varepsilon=0.05$ in order to justify the computational efforts of our algorithm, namely ensuring that our solution does not coincide with the explicit one produced by the EM algorithm for $\varepsilon=0$. 
	A detailed comparison with the EM and other classical algorithms, including convergence rates and performance tests, is beyond the scope of the present paper, and it will be addressed in a more computational oriented work.
	
 We first consider the case of Bernoulli mixtures, i.e. $S=2$, taking as dataset the MNIST database of handwritten digits \cite{url_MINST}, see Figure \ref{fig:bern_example_digits}. We recall that the database contains $60000$ images of the digits $\{\mathbf{0},\dots,\mathbf{9}\}$, each composed by $28\times 28$ pixels in $256$ grey levels, that we turn (via hard-threshold) in monochrome images and represent by binary vectors of size $D=784$.  Moreover, we remark that each sample in the database is already labelled by the number of the corresponding digit, that we use to check the correctness of the clusterization. To this end, given $1\le K\le 10$, we select $K$ digits $d_1,\dots,d_K$ in $\{\mathbf{0},\dots,\mathbf{9}\}$, and we run our algorithm to compute the Bernoulli parameters $\mu_k$ of the mixture components, for $k=1,\dots,K$. Then we the  introduce a matrix $H$ of size $K\times K$, whose entries are obtained as follows. For each sample $x$ of type $d_k$, we compute the corresponding responsibilities $\gamma_1(x),\dots,\gamma_K(x)$, and we accumulate these values in the $k$-th row of $H$, normalizing their sum with respect to the number of samples in $d_k$. More precisely, for $k,j=1,\dots,K$ 
	$$
	H_{kj}=\frac{1}{|d_k|}\sum_{x\in d_k}\gamma_j(x)
	$$
	provides the averaged probability, for a sample of type $d_k$, of belonging to the cluster $j$. Up to a permutation of the rows, namely a reordering of the mixture components, we can always assume that the maximal values  of $H$ correspond to the diagonal entries, so that digits of type $d_k$ will belong, with the highest probability, to the cluster $k$.  In a perfect clusterization, $H$ is clearly the identity matrix, but we recall that each Bernoulli distribution $\pi_k$ is built as a joint probability of all the observed pixel values \eqref{eq:pi_k}. Since the samples of type $d_k$ can be very different from each other, and also share some similarity with samples of other types (see Figure \ref{fig:nonhomo-digits}), we can never expect such a sharp partition. 
	
	For visualization purposes we represent each row of the matrix $H$ as a histogram, reporting on the $x$-axis the type $d_k$, for $k=1,\dots,K$, and assigning $K$ different colors to the values $H_{kj}$ for $j=1,\dots,K$, corresponding to the $K$ clusters $C_1,\dots,C_K$. Moreover, since the state space dimension is $S=2$, we can conveniently represent the parameters $\mu_k\in[0,1]^D$ of the corresponding Bernoulli distributions in the mixture, for $k=1,\dots,K$, as grey scale images. 
	\begin{figure}[!h]
		\centering
		\includegraphics[width =0.9\textwidth]{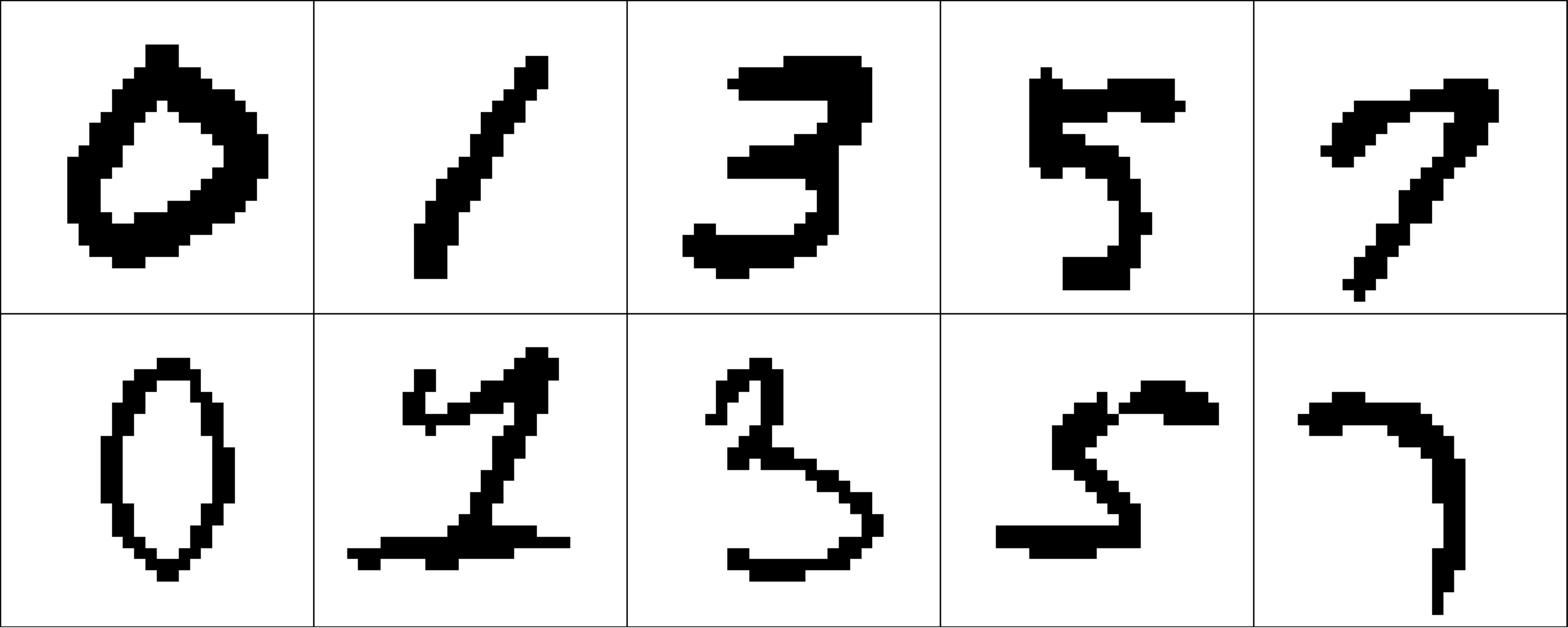}
		\caption{Different samples of hand-written digits from the MNIST database}
		\label{fig:nonhomo-digits}
	\end{figure}

	Let us start with the case $K=2$, choosing the digits $\mathbf{1},\mathbf{3}$. In Figure \ref{fig:test1}, we show the clusterization histogram and the corresponding Bernoulli parameters, observing that the samples of the two digits are very well separated. 
	\begin{figure}[!h]
		\centering
		\begin{tabular}{c}
			\includegraphics[width=0.75\textwidth]{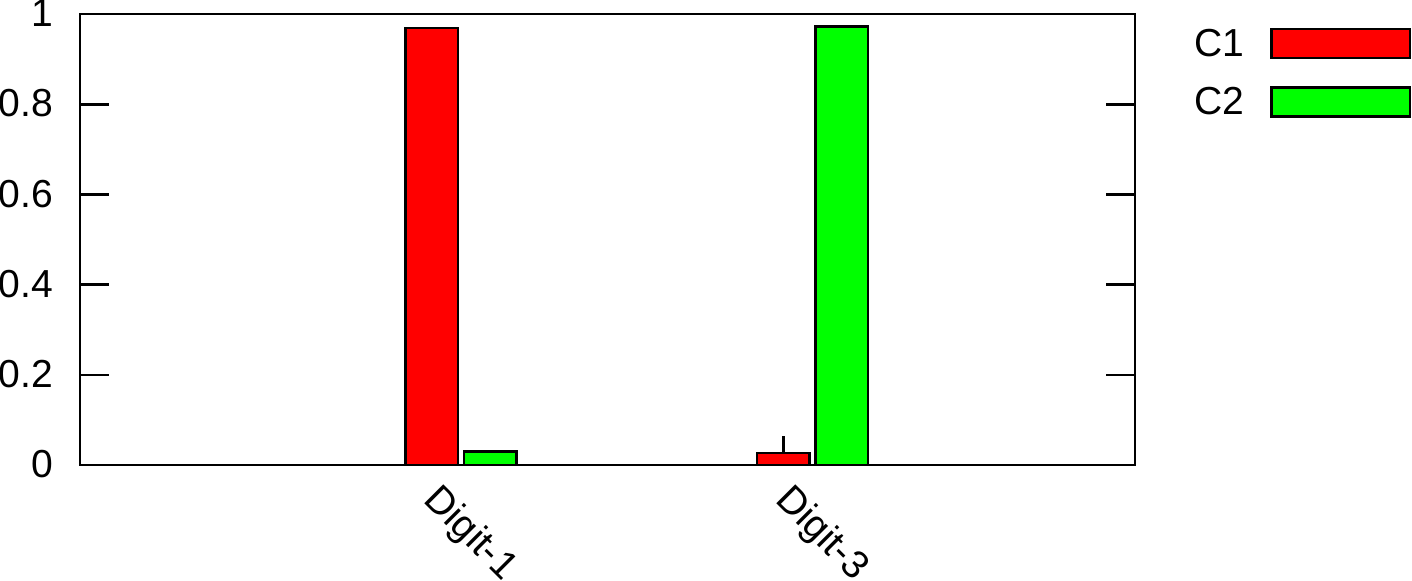}  \\
			\includegraphics[width=0.4\textwidth]{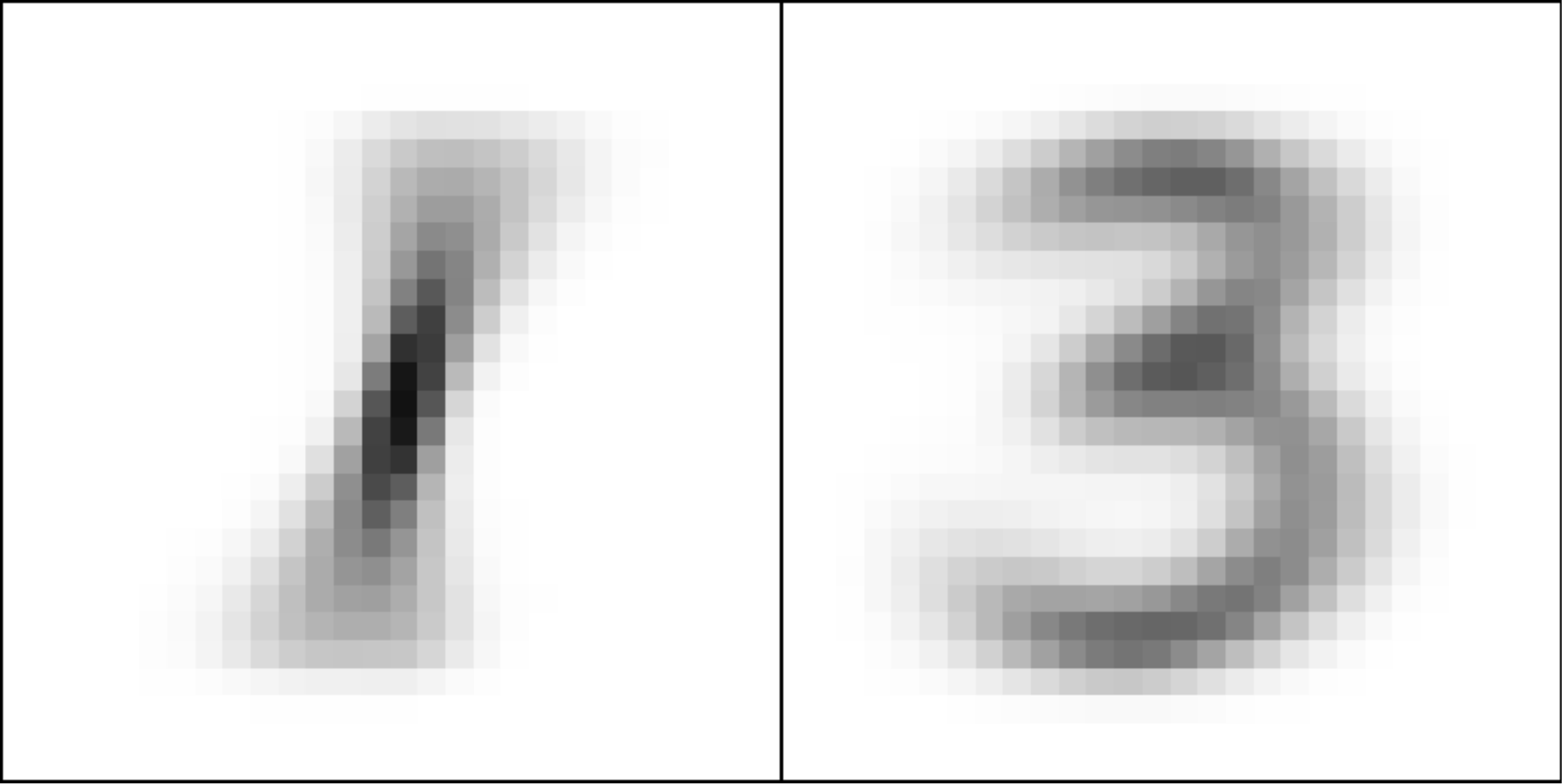}
		\end{tabular}
		\caption{Clusterization histogram for  digits $\mathbf{1},\mathbf{3}$ and the  corresponding Bernoulli parameters.}
		\label{fig:test1}
	\end{figure}
	
	We now choose the digits $\mathbf{3},\mathbf{5}$, and we show the results in Figure \ref{fig:test2}. In this case, we observe that the clusterization is slightly ambiguous, since, in average, the samples of the two types are more similar to each other (see the corresponding Bernoulli parameters). In the histogram we clearly see a repartition of about $60\%$ and $40\%$ (and viceversa) between the two clusters.
	\begin{figure}[!h]
		\centering
		\begin{tabular}{c}
			\includegraphics[width=0.75\textwidth]{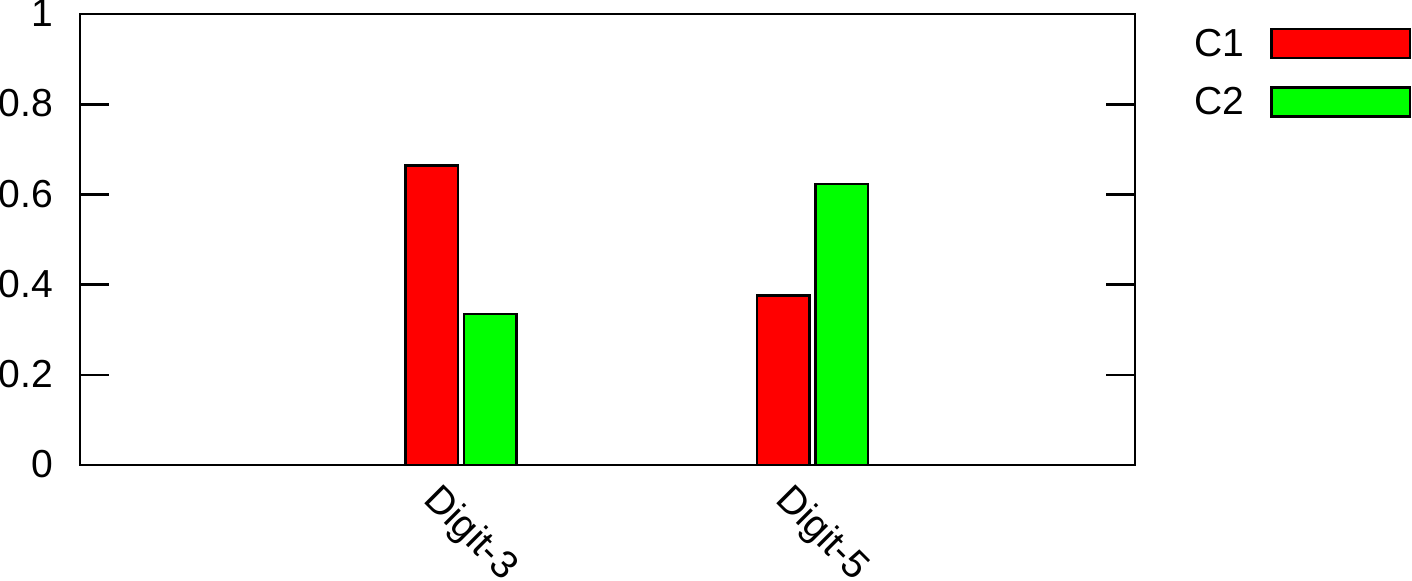}  \\
			\includegraphics[width=0.4\textwidth]{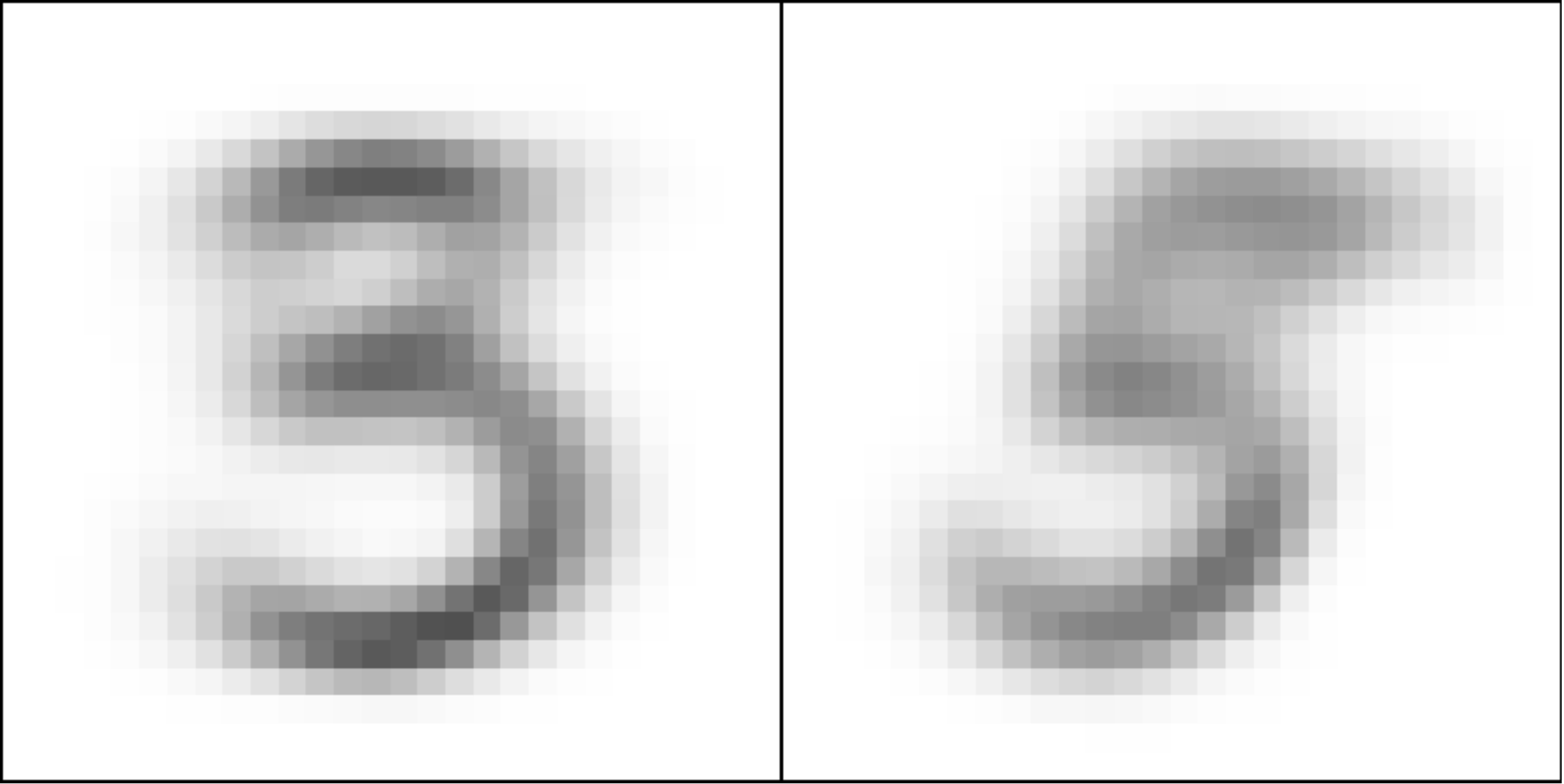}
		\end{tabular}
		\caption{Clusterization histogram for  digits $\mathbf{3},\mathbf{5}$ and the  corresponding Bernoulli parameters.}
		\label{fig:test2}
	\end{figure}
	
	We finally consider the case $K=5$ with even digits $\mathbf{0},\mathbf{2},\mathbf{4},\mathbf{6},\mathbf{8}$. In Figure \ref{fig:test3}, we observe that the chosen digits are, in average, different from each other, so that they are quite well clusterized. \begin{figure}[!h]
		\centering
		\begin{tabular}{c}
			\includegraphics[width=0.75\textwidth]{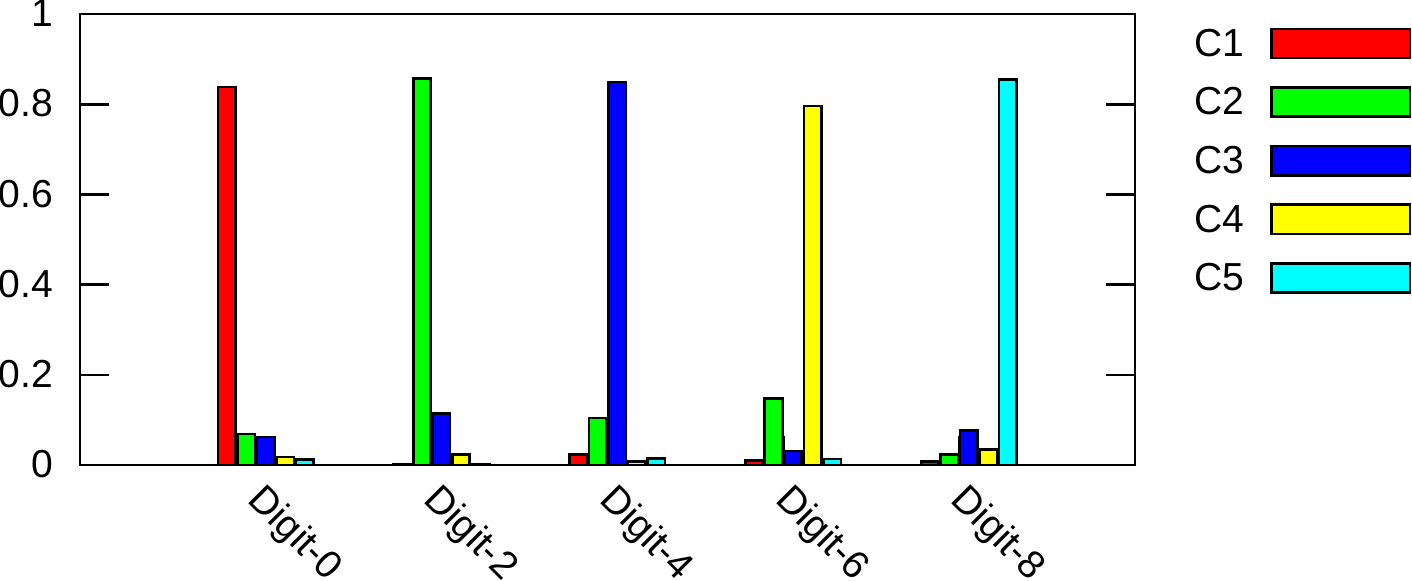}  \\
			\includegraphics[width=0.9\textwidth]{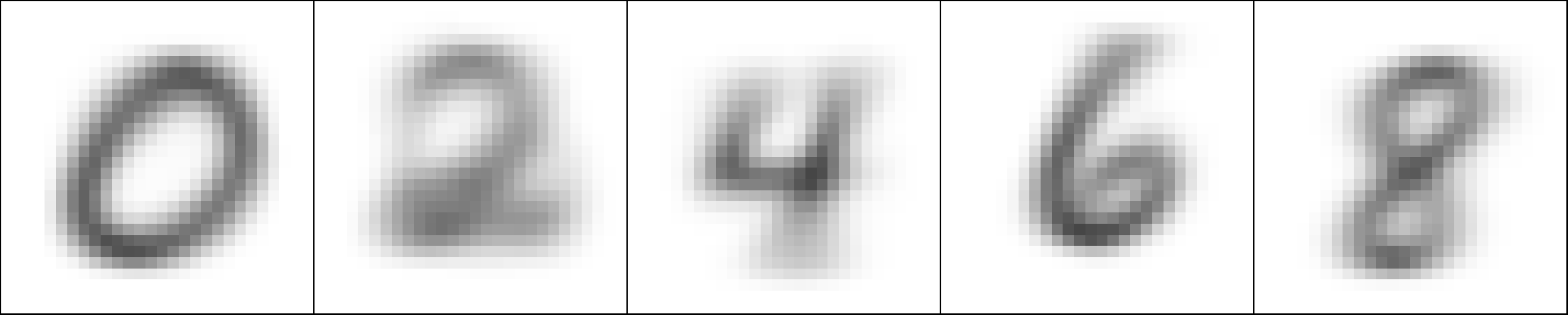}
		\end{tabular}
		\caption{Clusterization histogram for even digits and the  corresponding Bernoulli parameters.}
		\label{fig:test3}
	\end{figure}
	
	It is worth noting that the greatest error in the cluster assignment correspond to digit $\mathbf{6}$, which indeed shares about $15\%$ of its samples with the cluster of the digit $\mathbf{2}$. Looking at the corresponding Bernoulli parameters, we readily see that the two images have very similar vertical alignments, and also a quite large overlapping bottom region. In particular, the Bernoulli parameter for the digit $\mathbf{2}$ is visibly more diffused, and this reflects the inhomogeneity of the corresponding samples. Similar considerations also apply to the pairs of digits $\mathbf{2},\mathbf{4}$ and $\mathbf{4},\mathbf{8}$, with errors in the cluster assignment below $10\%$.       
	
	We now consider the case of categorial distributions, i.e. $S>2$, taking as dataset the Fashion-MNIST database \cite{url_FashionMINST}, see Figure \ref{fig:fashion-mnist}. 
	\begin{figure}[!h]
		\centering
		\includegraphics[width =0.9\textwidth]{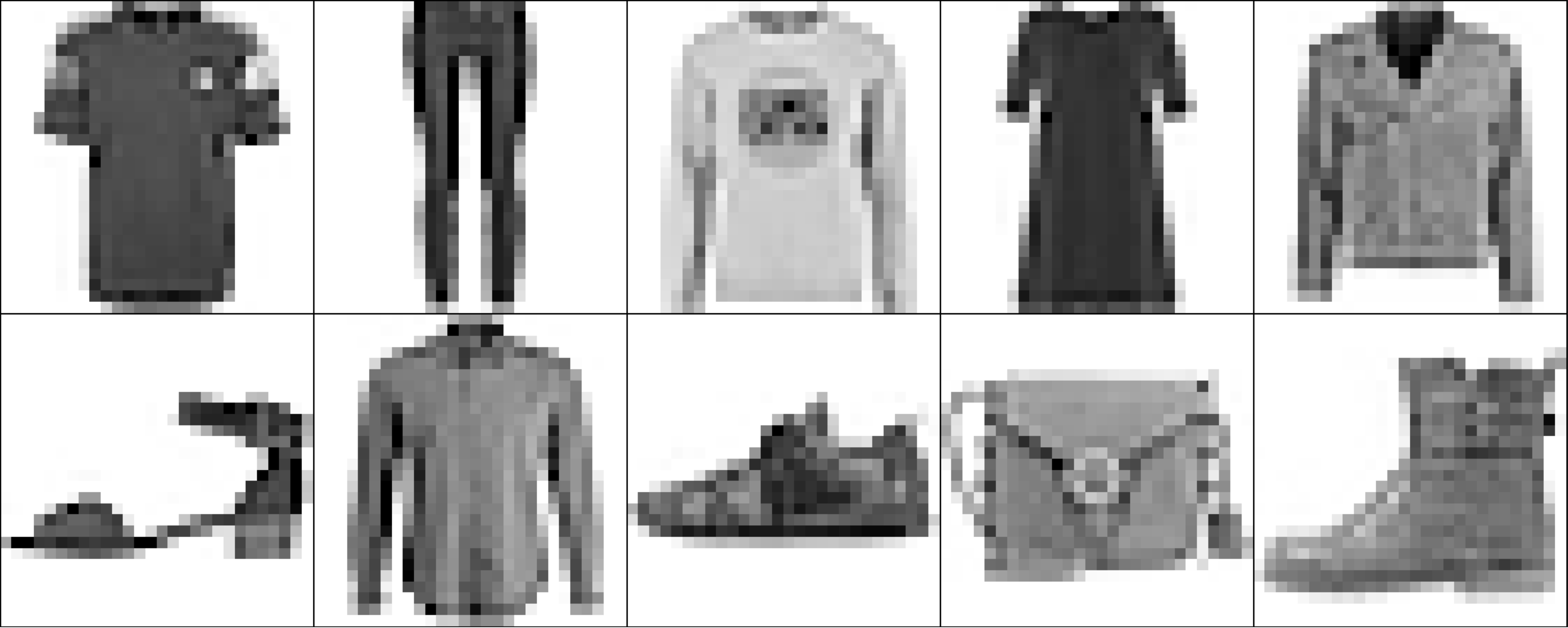}
		\caption{Samples of fashion products from the Fashion-MNIST database}
		\label{fig:fashion-mnist}
	\end{figure}
	
	The database contains $60000$ images of ten categories of fashion products, namely 
	$\{${\bf T-shirt}, {\bf Trouser}, {\bf Pullover}, {\bf Dress}, {\bf Coat}, {\bf Sandal}, {\bf Shirt}, {\bf Sneaker}, {\bf Bag}, {\bf Boot}$\}$, each composed by $28\times 28$ pixels in $256$ grey levels. We turn these images (via hard-thresholding) in images with $S$ grey levels, and represent them by vectors in $\mathcal{S}^D$, where $\mathcal{S}=\{1,\dots,S\}$ and $D=784$.  
	As in the previous tests, we use the label associated to each sample in the database to build the clusterization histogram. Moreover, we visualize the parameters $\theta_k(i)\in[0,1]^D$ of the corresponding categorical distributions in the mixture, for $i=1,\dots,S$ and $k=1,\dots,K$, as grey scale images, i.e. averaging the values with respect to $i$ as $\bar\theta_k=\frac{1}{S}\sum_{i=1}^S\theta_k(i)$. 
	
	In this example, clusterization is very challenging. To give an idea of the issues, we set $S=32$ and compute separately ten \textquotedblleft ideal" categorical distributions associated to the dataset, by simply averaging the pixel values of all the samples of a same type. The result is shown in Figure  \ref{fig:fashion-mnist-averaged}. 
	\begin{figure}[!h]
		\centering
		\includegraphics[width =0.9\textwidth]{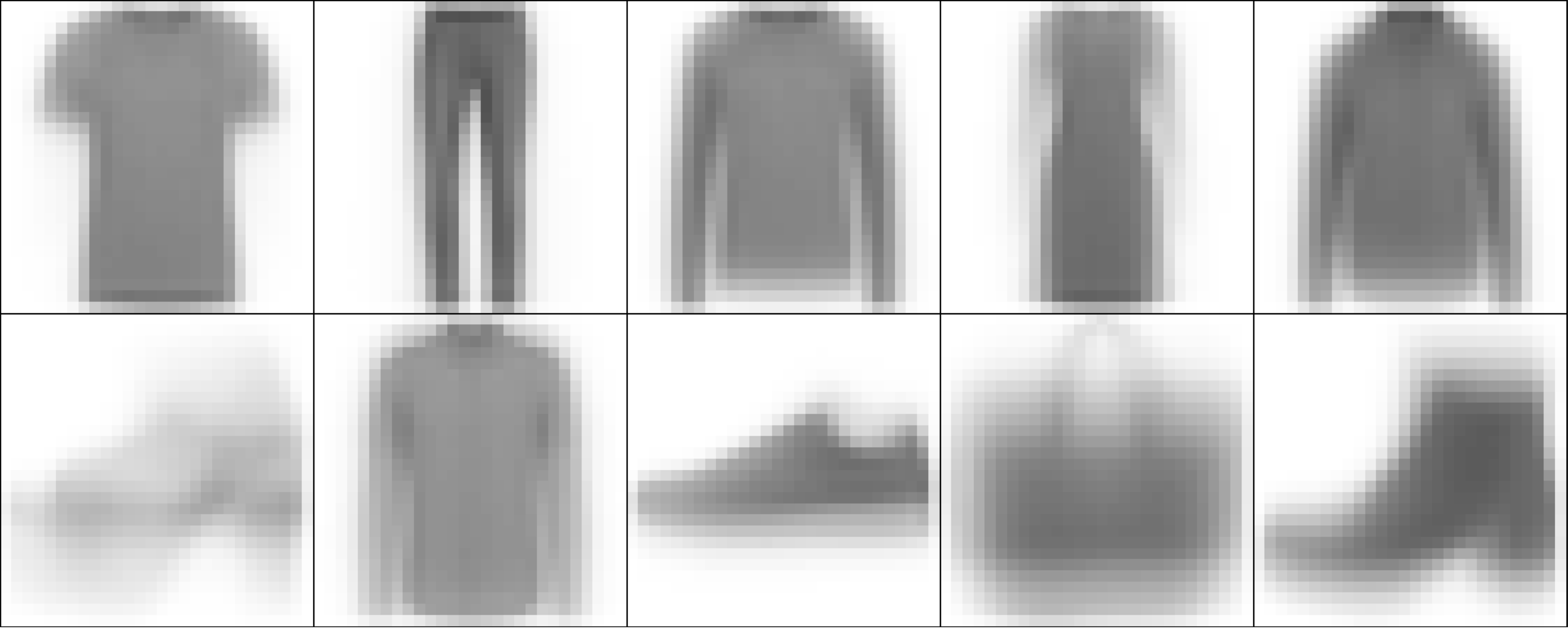}
		\caption{Averaged categorical distributions for the Fashion-MNIST database}
		\label{fig:fashion-mnist-averaged}
	\end{figure}
	
	We clearly see that types {\bf Pullover}, {\bf Coat} and {\bf Shirt} are almost indistinguishable. Moreover, they all have a very large overlap region with the type {\bf T-shirt}, but also with the types {\bf Trouser} and {\bf Dress}. Finally, we observe that the samples of type {\bf Sandal} are so different from each other that the corresponding distribution is completely smoothed out. As one can expect, such drawbacks dramatically affect the quality of the clusterization, but they also suggest how the present model could be improved, for instance including in the cost functions some geometric correlation between the image pixels. This direction of research is currently under development.
	
	We conclude this section with the following tests, which confirm the above considerations. We set $K=2$ and choose the types {\bf T-shirt} and {\bf Trouser}. In Figure \ref{fig:test4}, we show the resulting clusterization histogram and the corresponding categorical parameters. 
	\begin{figure}[!h]
		\centering
		\begin{tabular}{c}
			\includegraphics[width=0.75\textwidth]{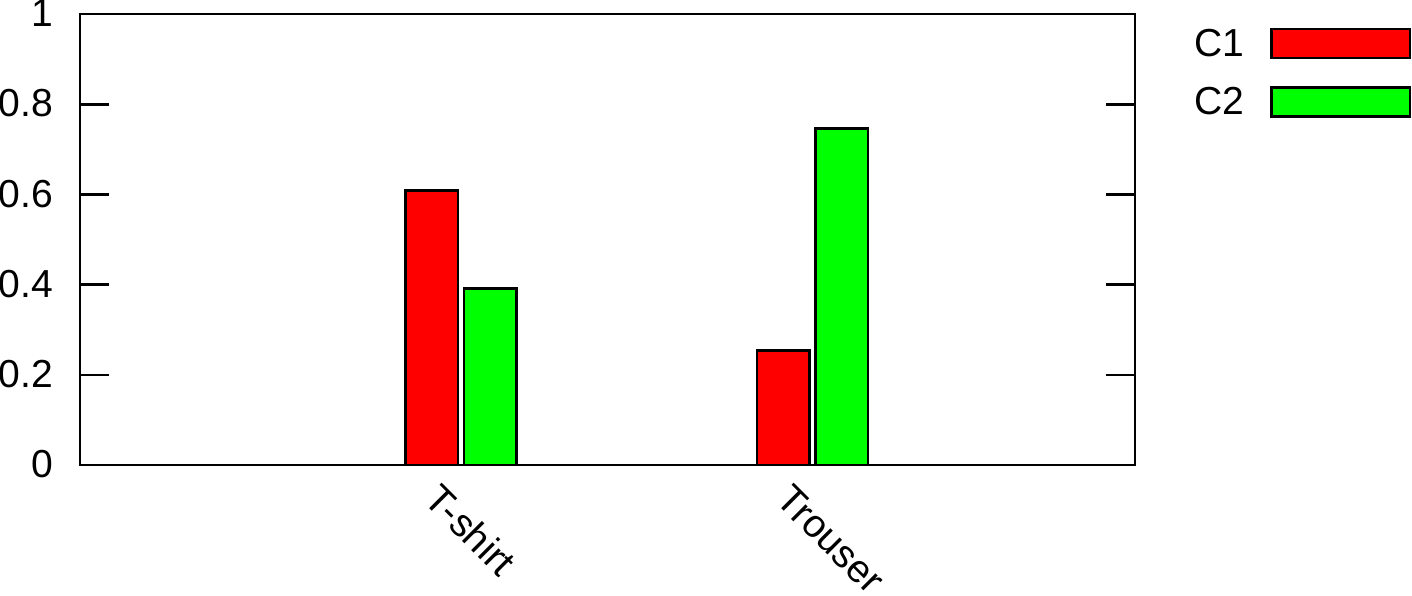}  \\
			\includegraphics[width=0.4\textwidth]{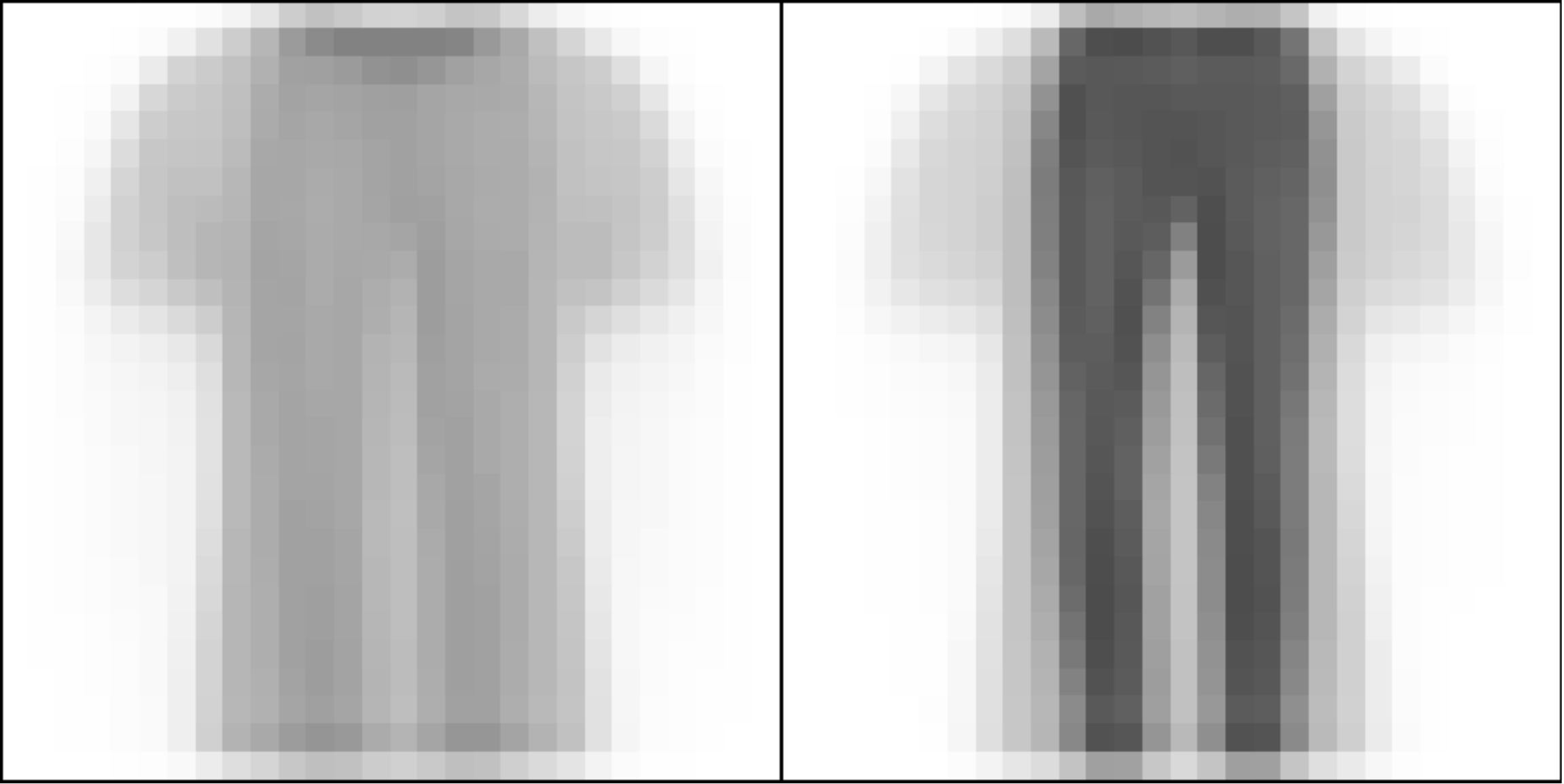}
		\end{tabular}
		\caption{Clusterization histogram for types {\bf T-shirt}, {\bf Trouser} and the  corresponding categorical parameters.}
		\label{fig:test4}
	\end{figure}
	
	Despite the clusterization being still \textquotedblleft acceptable\textquotedblright  (in average more than $60\%$ of samples correctly assigned to the corresponding clusters), we clearly observe a weird mixing of the two types. 
	On the other hand, choosing types which are substantially different from each other, we end up with a good clusterization. This is the case for the example shown in Figure \ref{fig:test5}, where we set $K=4$ and choose the types  {\bf Dress}, {\bf Sneaker}, {\bf Bag} and {\bf Boot}.
	\begin{figure}[!h]
		\centering
		\begin{tabular}{c}
			\includegraphics[width=0.75\textwidth]{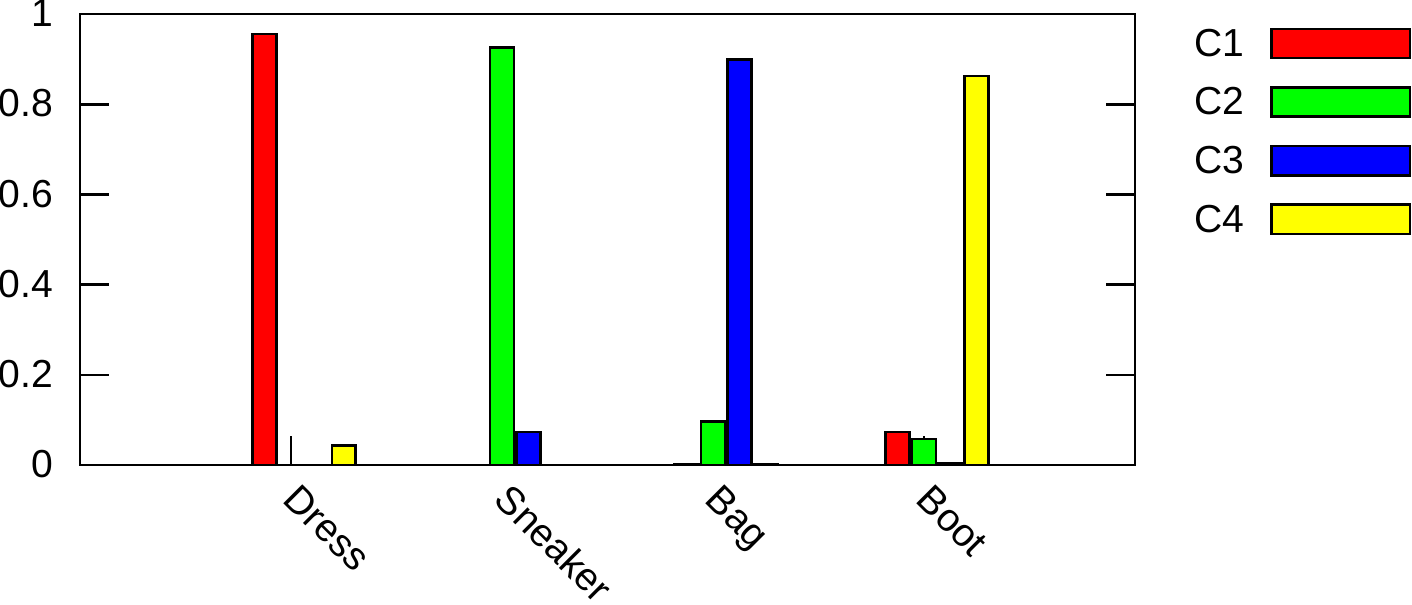}  \\
			\includegraphics[width=0.75\textwidth]{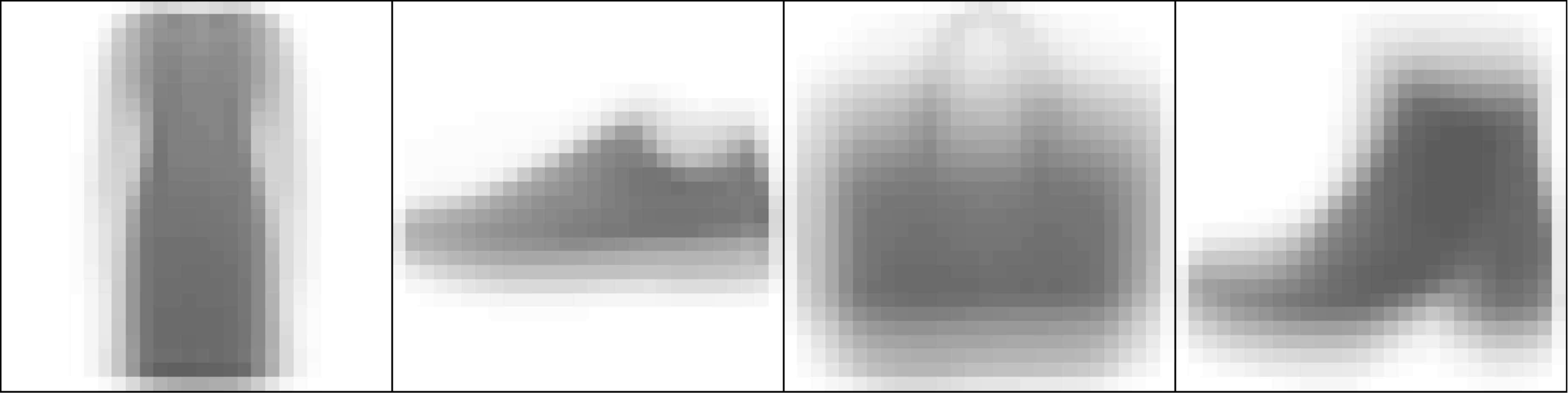}
		\end{tabular}
		\caption{Clusterization histogram for types 
			{\bf Dress}, {\bf Sneaker}, {\bf Bag}, {\bf Boot} and the  corresponding categorical parameters.}
		\label{fig:test5}
	\end{figure}
	%%%%%%%%%%%%%%%%%%%%%%%%%%%%%%%%%%%
	% 			Appendix              %
	%%%%%%%%%%%%%%%%%%%%%%%%%%%%%%%%%%%
	\newpage
	\appendix
	\section{A two-states stationary MFG system}\label{sec:appendix}
	In this section, we study some properties of a particular MFG system we need for the analysis performed   in Section \ref{sec:categorical}. In this case, the cost in the Hamilton-Jacobi-Bellman equation does not  depend on the distribution of the population, hence the  coupling term is only in the Fokker-Planck equation which characterizes the stationary distribution.  Here we mainly follow the notations in \cite{gms}. \\
	Given a matrix $P\in\maS$, we denote with $P_i$ the $i$-row of $P$.
	For $\e>0$, we consider  MFG system
	\begin{equation}\label{app:MFG}
		\left\{
		\begin{array}{ll}
			V(i)=\displaystyle{\min_{P_i:\,P_{ij}\ge 0, \sum_j P_{ij}=1}}\left\{\sum_{j=1}^S P_{ij}\left(c(P_{ij})+\e\log(P_{ij})+F(i, \theta)+V(j)\right)\right\}-\l\\ [8pt]
			\pi(i)= \sum_{j=1}^S P_{ji}\pi(j)\\[8pt]
			\pi(i)\ge 0,\, \sum_{i=1}^S\pi(i)=1, \,\sum_{i=1}^S V(i)=0
		\end{array}
		\right.
	\end{equation}
	for $i=1,\dots,S$, where $\e>0$ and the components of the vector  $\theta= (\theta(1), \dots, \theta(S))\in \mathbb{S}$  are fixed parameters and the transition matrix $P$  in the Fokker-Planck equation is composed of the rows $$P_i=\{P_{ij}\}_{j=1}^S$$ which realize  the minimum in the Hamilton-Jacobi-Bellman equation.
	The transition  cost from the state $i$ to the state $j$ is  given by
	\begin{equation*}%\label{app:trans_cost}
		C^\e(P_{ij},\theta)=c(P_{ij})+\e\log(P_{ij}) +F(i, \theta) ,\qquad i,j=1,\dots,S.
	\end{equation*}
	We assume that the cost function $c\in C^1([0,1])$ with $pc(p)$  convex  for $p\in [0,1]$  and $F:\mathcal{S}\times \mathbb{S}\rightarrow \mathbb{R}$ is such that $F(i, \cdot)$ is bounded and continuous for all $ i  \in \mathcal{S}$.\\
	The average cost in the state $i$ for a given choice of the transition matrix $P$ is defined by
	\begin{equation*}%\label{app:average_cost}
		e_i(P,V)= \sum_{j=1}^S \left(C^\e(P_{ij},\theta)+V(j)\right)P_{ij} \qquad i=1,\dots,S, 
	\end{equation*}
	and we denote with $e(P,V)$ the corresponding $S$-dimensional vector. Note that $e_i(P,V)$ depends only on the $i$-th row of the matrix $P$.\\
	We introduce the definition of Nash minimizer for $e(P,V)$ (see \cite[Definition 1]{gms}). Given the stochastic matrix $P\in \maS$ and a probability vector $q\in \mS$, we denote with $\mR(P,q,i)$   the stochastic matrix obtained by replacing the $i$-th row of $P$ with the vector $q$.
	\begin{definition}\label{app:def_nash_min}
		Given a cost vector $V\in\R^S$, a stochastic matrix $P$ is said to be a Nash minimizer for $e(P,V)$ if for each $i=1,\dots,S$, $q\in\mS$, it holds
		\[e_i(P,V)\le e_i(\mR(P,q,i),V).\]
	\end{definition}
	%%%%
	
	\begin{proposition}\label{app:lemma_existence_Nash} 
		For each vector $V\in\R^S$, there exists a unique Nash minimizer $P$ for $e(P,V)$.
	\end{proposition}
	\begin{proof}
		As proved in \cite[Theorem 1]{gms}, existence of a Nash minimizer $P$ follows by the continuity and convexity of the cost $e_i(P,V)$ with respect to the vector $P_i$, $i=1,\cdots,S$.
		Moreover, a straightforward computation gives that the function $g:\mS^S\to \mS^S$, defined by $g_{ij}(P)=\partial e_i(P,V)/\partial P_{ij}$, is diagonally convex, i.e. for all $P^1$, $P^2\in\maS$, $P^1\neq P^2$, it holds
		\[\sum_{i,j}(P^1_{ij}-P^2_{ij})(g_{ij}(P^1)-g_{ij}(P^2))>0.
		\]
		Indeed,  for any $i, j \in \{1, \dots, S\}$ and $P \in \maS$, we have 
		\begin{align*}
			\frac{\partial e_i(P, V)}{\partial P_{ij}} 
			%&= \sum_{l = 1}^{S} P_{il} \frac{\partial}{\partial P_{ij}}\left( C^{\e} (P_{il},\theta) + V(l) \right) + \left( C^{\e} (P_{ij},\theta) + V(j)\right)\\
			&=P_{ij}(c'(P_{ij}) + \e \frac{1}{P_{ij}}) + c_{ij}(P_{ij})+\e\log(P_{ij}) + V(j)+ F(i,\th).
		\end{align*}
		Hence, 
		\begin{align*}(P^1_{ij} - P^2_{ij}) \cdot &\left(g_{ij}(P^1) - g_{ij}(P^2)\right) = 
			\e (P^1_{ij} - P^{2}_{ij})\log\left(\frac{P^1_{ij}}{P^2_{ij}}\right)\\&+  P^1_{ij}c(P^1_{ij})-P^2_{ij}c(P^2_{ij})-(P^2_{ij}c'(P^2_{ij})+c(P^2_{ij}))(P^1_{ij}-P^2_{ij})\\
			&+P_{ij}^2c(P^2_{ij})-P^1_{ij}c(P^1_{ij})-(P^1_{ij}c'(P^1_{ij})+c(P^1_{ij}))(P^2_{ij}-P^1_{ij}))>0,
		\end{align*} 
		because of the monotonicity of $\log(p)$ and the convexity  of $pc(p)$ for $p\in[0,1]$. By	\cite[Theorem 2]{gms}, we get the uniqueness of the Nash minimizer.
	\end{proof}
	We now prove that  $\pi$ has positive mass   for any state $i=1,\dots,S$. This is a crucial result for the existence of a solution to the multi-population MFG system studied in Section \ref{sec:categorical}. We need some preliminary results.
	\begin{lemma}\label{app:lemma_HJB}
		For any vector $\theta\in \mS$, there exists a unique solution $(\l_\th,V_\th)$ to the Hamilton-Jacobi-Bellman equation
		\begin{equation}\label{app:HJB}
			\left\{
			\begin{array}{ll}
				V(i)=\displaystyle{\min_{P_i:\,P_{ij}\ge 0, \sum_j P_{ij}=1}}\left\{\sum_{j=1}^S \left(C^\e(P_{ij},\theta)+V(j)\right)P_{ij}\right\}-\l\\ [8pt]
				\sum_{i=1}^S V(i)=0.
			\end{array}
			\right.
		\end{equation}
		Moreover, if $\th\to\bar\th$, then $(\l_\th,V_\th)$
		converges to  $(\l_{\bar\th},V_{\bar\th})$.
	\end{lemma}
	\begin{proof}
		Existence and uniqueness of a solution to \eqref{app:HJB} follows from \cite[Proposition 8 and Theorem 4]{gms}. Indeed, since the cost in the Hamilton-Jacobi-Bellman equation is independent of the distribution $\pi$, the assumptions of Proposition 8 and Theorem 4 in \cite{gms} are trivially satisfied. The convergence of  $(\l_\th,V_\th)$ to  $(\l_{\bar\th},V_{\bar\th})$ follows by the continuity of the costs $e_i(P,V)$ with respect to $\theta$ and the uniqueness of the solution to \eqref{app:HJB}.
	\end{proof}
	%%%%
	Proposition \ref{app:lemma_existence_Nash} and the continuity of the average cost $e(P,V)$ implies the continuity  of the Nash minimizer (see \cite[Proposition 1]{gms}).
	\begin{lemma}\label{app:lemma_continuity_Nash}
		The  Nash minimizer $P(V)$ is a continuous function of $V$.
	\end{lemma}
	%%%%
	\begin{theorem}\label{app:theorem_positivity}
		For any vector $\theta\in \mS$, there exists a unique solution $(V_\th,\l_\th,\pi_\theta)\in \R^S\times\R\times\mS$ to \eqref{app:MFG}. Moreover, there exists a positive constant $c(\e)\in (0,1)$ such that, for each $\theta\in \mS$ and for each $i=1,\dots,S$
		\begin{equation}\label{app:positivity_measure}
			\pi_\th(i)\ge c(\e)>0.
		\end{equation}
	\end{theorem}
	\begin{proof}
		Note that, because of the particular structure of the system \eqref{app:MFG}, Propositions \ref{app:lemma_existence_Nash} and \ref{app:lemma_HJB}   immediately ensure existence and uniqueness of a solution to \eqref{app:MFG}. \\
		To show \eqref{app:positivity_measure}, we argue by contradiction assuming that there exists a sequence $\th_n\in \mathbb{S}$ such that $\pi_{\th_n}(i_n)\to 0$ for $n\to \infty$. Since the state space is finite we assume that, up to a subsequence, $i_n=i$ for any $n$. Let $\th\in \mathbb{S}$ be such that, up to a subsequence, $\th_n\to \th$. By Lemma \ref{app:lemma_HJB} and \ref{app:lemma_continuity_Nash}, it follows that $P(V_{\th_n})\to P(V_{\th})$, where $P(V_{\th_n})$ and $P(V_{\th})$ are the Nash minimizers   corresponding to  the solution of \eqref{app:HJB} with $\th_n$ and, respectively, with $\th$. By the equation
		\[\pi_{\th_n}=\pi_{\th_n}P(V_{\th_n}),\]
		satisfied by the invariant distribution associated to
		$P(V_{\th_n})$, we get that, up to a subsequence, $\pi_{\th_n}$ converges to $\bar \pi\in \mS$ satisfying
		\[\bar\pi =\bar\pi P(V_{\th}),\]
		with $\bar\pi(i)=0$. But, because of entropy regularization term, all the entries $P_{ij}$  of the transition matrix $P(V_{\th})$  are strictly  positive, hence $P$ is a positive matrix and   the associated invariant distribution satisfies $\bar\pi(i)>0$ for each $i=1,\dots,S$. By the contradiction, we obtain \eqref{app:positivity_measure}.
	\end{proof}
	%
	%\begin{remark}
	%\textcolor{red}{
	%The transition cost $c^\e_{ij}$ in \eqref{app:trans_cost} can be thought as the combination of  the   costs $c^1_{ij}=-(1-P_{ij})+|i-\theta|^2/2$
	%and $c^2_{ij}=\e ln (P_{ij})+|i-\theta|^2/2$. If $S=1$,
	%the  invariant measures corresponding to the two previous costs can easily computed   and  they are two binomial distributions of parameter $\mu^1=\theta$
	%and 
	%\[\mu^2= \dfrac{1+ e^{-\frac{2\theta -1}{2\e}}}{2(1+ \cosh(\frac{2\theta -1}{2 \e}))}.\]
	%Note that $\mu^2>0$ and  $\mu^2=1/2$   if $\th=1/2$. Moreover, for $\e\to 0$,  $\mu^2\to 0$ if $\th>  1/2$,  $\mu^2\to 1$ if $\th<1/2$.}
	%\end{remark}
	
	%\begin{figure}[!h]
	%\centering
	%\begin{tikzpicture}
	%  \matrix (m) [matrix of math nodes,row sep=6em,column sep=8em,minimum width=3em]
	%  {
	%      V & P = P(V) \\
	%    \theta & \pi \\
	%    };
	%  \path[-stealth]
	%    (m-2-1) edge node [left] {HJB: Lemma $\eqref{app:lemma_HJB}$} (m-1-1)
	%    (m-1-1)   edge node [above] {Lemma $\eqref{app:lemma_continuity_Nash}$} (m-1-2)
	%    (m-2-2) edge node [below] {} (m-2-1)
	%    (m-1-2) edge node [right] {FP} (m-2-2);
	%    
	%\end{tikzpicture}
	%\caption{The diagram shows the dependance between the equations in \eqref{app:MFG}.}
	%\end{figure}
	
	%%%%%%%%%%%%%%%%%%%%%%%%%%%%%%%
	%       Bibliografia          %
	%%%%%%%%%%%%%%%%%%%%%%%%%%%%%%%

	%%%%%%%%%%%%%%%%%%%%%%%%%  Address %%%%%%%%%%%%%%%%%%%%%%%%%%%%
	\medskip
	
	\begin{flushright}
		\noindent \verb"laura.aquilanti@sbai.uniroma1.it"\\
		\noindent \verb"fabio.camilli@sbai.uniroma1.it"\\
		SBAI, Sapienza Universit\`{a} di Roma\\
		via A.Scarpa 14, 00161 Roma (Italy)	
	\end{flushright}
	
	\begin{flushright}	
		\noindent \verb"cacace@mat.uniroma3.it"\\
		Dipartimento di Matematica e Fisica\\
		Universit\`{a} degli Studi Roma Tre\\
		Largo S. L. Murialdo 1, 00146   Roma (Italy)
	\end{flushright}
	
	\begin{flushright}
		\noindent \verb"r.demaio@iconsulting.biz"\\
		IConsulting\\
		Via della Conciliazione 10, 00193  Roma (Italy)	
	\end{flushright}

\end{document}